\documentclass[letterpaper]{article} 
\usepackage{aaai24}  
\usepackage{times}  
\usepackage{helvet}  
\usepackage{courier}  
\usepackage[hyphens]{url}  
\usepackage{graphicx} 
\usepackage{natbib}  
\usepackage{caption} 
\usepackage{algorithm}
\usepackage{algorithmic}
\usepackage{newfloat}
\usepackage{listings}
\DeclareCaptionStyle{ruled}{labelfont=normalfont,labelsep=colon,strut=off} 
\floatstyle{ruled}
\newfloat{listing}{tb}{lst}{}
\floatname{listing}{Listing}
\title{Kernelized Normalizing Constant Estimation: \\ Bridging  Bayesian Quadrature and Bayesian Optimization}
\author{
    Xu Cai\textsuperscript{\rm 1},
    Jonathan Scarlett\textsuperscript{\rm 1,2}
}
\affiliations{
    \textsuperscript{\rm 1} Department of Computer Science, National University of Singapore\\
    \textsuperscript{\rm 2} Department of Mathematics, Institute of Data Science, National University of Singapore\\
    caix@u.nus.edu,
    scarlett@comp.nus.edu.sg
    
}

\usepackage{bibentry}

\usepackage[T1]{fontenc}
\usepackage[utf8]{inputenc}
\usepackage{amsthm}
\usepackage{amsmath}
\usepackage{amssymb}
\usepackage{graphicx}
\usepackage{dsfont}
\usepackage{amsmath}
\usepackage{array}
\usepackage{multirow}
\usepackage{caption}
\usepackage{color}
\usepackage{ifthen}

\usepackage{multicol}

\theoremstyle{plain}

\theoremstyle{plain}

\theoremstyle{plain}
\newtheorem{lem}{\protect\lemmaname}
\theoremstyle{plain}
\newtheorem{thm}{\protect\theoremname}
\theoremstyle{plain}
  
\theoremstyle{definition}

\theoremstyle{definition}

\theoremstyle{definition}

\makeatother
  
\usepackage{babel} 

\providecommand{\claimname}{Claim}
\providecommand{\lemmaname}{Lemma}
\providecommand{\propositionname}{Proposition}
\providecommand{\theoremname}{Theorem}
\providecommand{\corollaryname}{Corollary} 
\providecommand{\definitionname}{Definition}
\providecommand{\assumptionname}{Assumption}
\providecommand{\remarkname}{Remark}

\DeclareMathOperator*{\argmax}{arg\,max}

\newboolean{showedits}
\setboolean{showedits}{true} 
\ifthenelse{\boolean{showedits}}
{
 \newcommand{\del}[1]{\textcolor{red}{\sout{#1}}} 
}{
 \newcommand{\del}[1]{} 
 
}

\newboolean{showcomments}
\setboolean{showcomments}{true} 

\ifthenelse{\boolean{showcomments}}
{\newcommand{\nbc}[3]{
 {\colorbox{#3}{\bfseries\sffamily\scriptsize\textcolor{white}{#1}}}
 {\textcolor{#3}{\sf\small$\blacktriangleright$\textit{#2}$\blacktriangleleft$}}}
 }
{\newcommand{\nbc}[3]{}
 \renewcommand{\del}[1]{} 
 }

\definecolor{tdcolor}{rgb}{1.0,0,0}
\definecolor{rplcolor}{rgb}{0,0.0,1.0}

\newcommand{\GP}{{\rm GP}}

\newcommand{\balpha}{\boldsymbol{\alpha}}

\newcommand{\Hc}{\mathcal{H}}

\newcommand{\Zhat}{\hat{Z}}

\newcommand{\Nc}{\mathcal{N}}

\newcommand{\xvhat}{\hat{\mathbf{x}}}

\newcommand{\xv}{\mathbf{x}}

\newcommand{\Xv}{\mathbf{X}}
\newcommand{\yv}{\mathbf{y}}

\newcommand{\EE}{\mathbb{E}}
\newcommand{\PP}{\mathbb{P}}
\newcommand{\RR}{\mathbb{R}}

\newcommand{\var}{\mathrm{Var}}

\newcommand{\Bv}{\mathbf{B}}

\newcommand{\Kv}{\mathbf{K}}
\newcommand{\kv}{\mathbf{k}}

\newcommand{\bzero}{\boldsymbol{0}}

\newcommand{\bxi}{\boldsymbol{\xi}}

\newcommand{\fhat}{\hat{f}}

\newcommand{\Rhat}{\hat{R}}

\newcommand{\mat}{\mathcal{H}^{\nu}}

\newcommand{\ahat}{\hat{a}}

\newcommand{\beps}{\boldsymbol{\epsilon}}
\newcommand{\Tanh}{\mathrm{Tanh}}

\newcommand{\xhat}{\hat{x}}

\usepackage{subcaption}
\newcommand{\sizecustom}{0.27\linewidth}
\newboolean{ARXIV}
\setboolean{ARXIV}{false}

\begin{document} 


\maketitle

\begin{abstract}
In this paper, we study the problem of estimating the normalizing constant $\int e^{-\lambda f(\xv)}d\xv$ through queries to the black-box function $f$, where $f$ belongs to a {\em reproducing kernel Hilbert space} (RKHS), and $\lambda$ is a problem parameter.  We show that to estimate the normalizing constant within a small relative error, the level of difficulty depends on the value of $\lambda$: When $\lambda$ approaches zero, the problem is similar to {\em Bayesian quadrature} (BQ), while when $\lambda$ approaches infinity, the problem is similar to {\em Bayesian optimization} (BO). More generally, the problem varies between BQ and BO.  We find that this pattern holds true even when the function evaluations are noisy, bringing new aspects to this topic.  Our findings are supported by both algorithm-independent lower bounds and algorithmic upper bounds, as well as simulation studies conducted on a variety of benchmark functions.

\end{abstract}
\section{Introduction}
The problem of {\em normalizing constant (NC) estimation} (also known as {\em (log-)partition function estimation}) is of interest in a variety fields, such as Bayesian statistics \citep{Che97, Gel98}, machine learning \citep{Des11}, statistical mechanics \citep{Sto10}, and other areas involving the distribution of an energy function.  Given a distribution on a domain $D$ with measure $d\xv$, the normalizing constant is the integral $Z=\int_D e^{-f(\xv)}d\xv$.  In many classical works, $f$ is assumed to be (strongly) convex, so that the distribution is (strongly) log-concave.  Asymptotic/non-asymptotic performance bounds for the log-concave setting have been studied in detail; see for example \citep{Ge20} and the references therein.

Although the first-order (gradient) information of $f$ is usually assumed to be available in the classical setting, it is also of interest to estimate the normalizing constant of non-convex black-box functions using only zeroth-order bandit feedback.  Following the rich literature on Bayesian optimization (BO) and Bayesian quadrature (BQ), it is particularly natural to consider functions lying in a {\em reproducing kernel Hilbert space} (RKHS), for which Gaussian process (GP) techniques can be used to quantify uncertainty.  The NC estimation problem brings new challenges compared to BO (which seeks to find $\argmax_{\xv\in D}f(\xv)$) and BQ (which seeks to approximate $\int_D f(\xv) d\xv$), and in fact turns out to share features of both.


In this paper, we study the asymptotic limits of NC estimation for functions in RKHS with bounded norm, adopting Bayesian numerical analysis techniques.  As is commonly done, we consider an additional parameter $\lambda$ in the estimation problem, i.e., we consider $Z=\int_D e^{-\lambda f(\xv)}d\xv$.  The interpretation of $\lambda$ differs across different subjects; for instance, it can represent the reciprocal of the thermodynamic temperature of a system, or it can be used to ``temper'' or ``amplify'' terms in Bayesian statistics.  Additionally, we consider the scenario where observations may be corrupted by additive Gaussian noise with variance $\sigma^2$, and $\lambda$ and $\sigma$ may vary with the time horizon $T$ as $T\to\infty$.  

\begin{figure}[t]
    \centering
    \ifthenelse{\boolean{ARXIV}}
    {
        \includegraphics[width=0.7\linewidth]{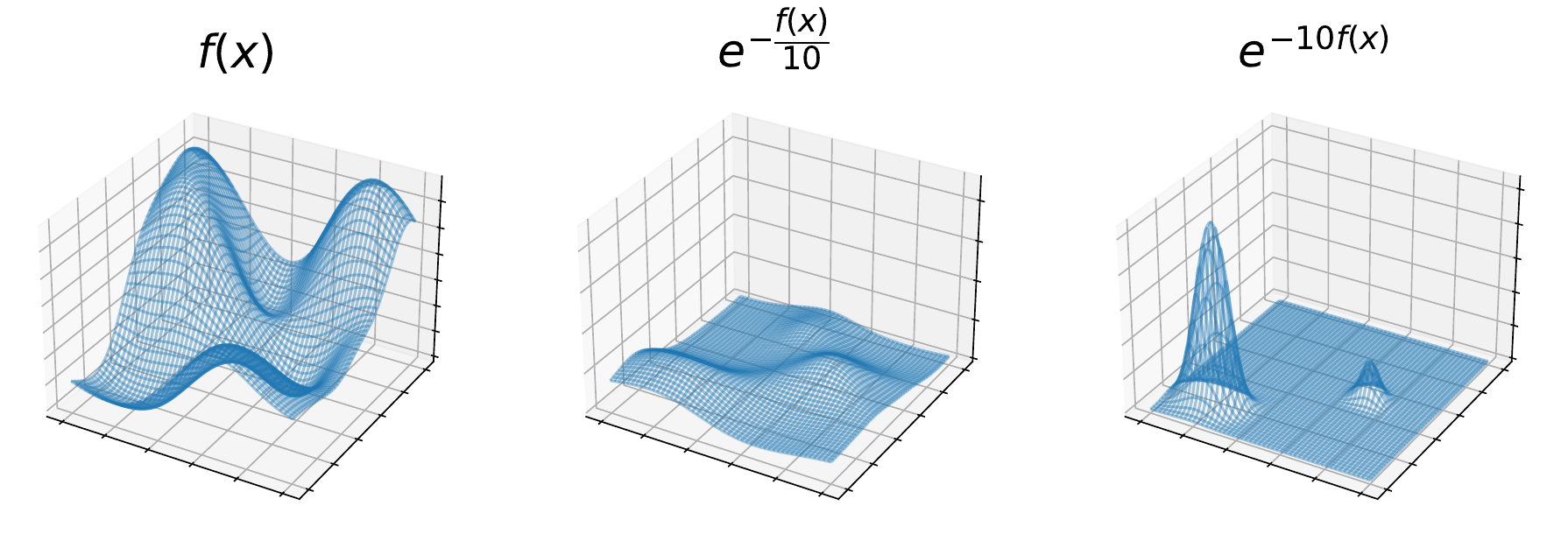}
    }{
        \includegraphics[width=\columnwidth]{figs/demo.pdf}
    }
    \caption{An illustration of estimating the normalizing constant under small (0.1) and large (10) $\lambda$.}
    \label{fig:demo}
\end{figure}

\begin{table*}[t]
    \def\arraystretch{1.85}
    \centering
    \resizebox{0.75\textwidth}{!}{
    \begin{tabular}{|cccc|cccc|}
    \hline
    \multicolumn{4}{|c|}{$\sigma=0$ (noiseless)}   &  \multicolumn{4}{c|}{$\sigma=\Theta(T^{-\frac{1}{4}})$}     \\ \hline
    \multicolumn{2}{|c|}{$\lambda\to 0$}   & \multicolumn{2}{c|}{$\lambda\to\infty$}     & \multicolumn{2}{c|}{$\lambda\to 0$}  & \multicolumn{2}{c|}{$\lambda\to\infty$}    \\ \hline
    \multicolumn{2}{|c|}{$\Theta(\lambda T^{-\frac{\nu}{d}-1})$}   & \multicolumn{1}{c|}{$\Omega(\lambda T^{-\frac{\nu}{d}-1})$} & \multicolumn{1}{c|}{$O(\lambda T^{-\frac{\nu}{d}-\frac{1}{2}})$} & 
    \multicolumn{4}{c|}{$\Theta(\lambda T^{-\frac{3}{4}})$ (needs $\nu\ge \frac{d}{2}$ when $\lambda\to\infty$)}  \\ 
    
    \hline
    \multicolumn{8}{c}{} \\[-1.8em]
    \hline

    \multicolumn{4}{|c|}{$\sigma=\Theta(T^{-\frac{1}{2}})$} &  \multicolumn{4}{c|}{$\sigma=\Theta(1)$ (constant)}  \\ \hline
    \multicolumn{2}{|c|}{$\lambda\to 0$}  & \multicolumn{2}{c|}{$\lambda\to\infty$} & \multicolumn{2}{c|}{$\lambda\to 0$}  & \multicolumn{2}{c|}{$\lambda\to\infty$}  \\ \hline
    \multicolumn{2}{|c|}{$\Theta(\lambda T^{-1})$} & \multicolumn{1}{c|}{$\Omega(\lambda T^{-1})$} & \multicolumn{1}{c|}{$O^*\big(\lambda T^{-\frac{4\nu+d}{4\nu+2d}}\big)$} &
    \multicolumn{2}{c|}{$\Theta(\lambda T^{-\frac{1}{2}})$} & \multicolumn{1}{c|}{$\Omega(\lambda T^{-\frac{1}{2}})$} & \multicolumn{1}{c|}{$O^*\big(\lambda T^{-\frac{\nu}{2\nu+d}}\big)$}  \\ \hline
    \end{tabular}
    }
    \caption{Selected results instantiated from Theorems \ref{thm:noiseless_lower}, \ref{thm:noisy_lower}, \ref{thm:noiseless_upper} and \ref{thm:noisy_upper} for NC.  $O^*(\cdot)$ hides $\mathrm{poly}(\log T)$ terms.  $\Theta(\cdot)$ means the lower and upper bounds match up to constant factors.  For $\lambda\to \infty$, we treat $\lambda=\Theta(T^c)$, $c>0$ and $T\to\infty$.  \label{tab:results}}
\end{table*}

For $d$-dimensional Mat\'ern-$\nu$ RKHS functions, our findings reveal that, to estimate $Z$ within a multiplicative factor of $1\pm\epsilon$ with constant probability (e.g., 0.99), the level of difficulty generally exhibits the following behaviour (also partially listed in Table \ref{tab:results}):
\begin{itemize}
    \item When $\lambda\to 0$, the error bound of $\epsilon$ is similar to BQ.  Our lower bound is similar to the BQ lower bounds stated in \citep{Pla96, Cai23}, who showed that the order-optimal average {\em mean absolute error} of $\epsilon$ is $\Theta(T^{-\frac{\nu}{d}-1}+\sigma T^{-\frac{1}{2}})$.
    \item When $\lambda \to\infty$, the problem becomes more similar to BO, and the error bounds again reflect this.  For instance, when $\lambda=\Theta(T)$ and $T\to\infty$, our noiseless $\Omega(T^{-\frac{\nu}{d}})$ lower bound coincides with the noiseless BO lower bound stated in \citep{Bul11}, who showed that the average {\em simple regret} is $\Theta(T^{-\frac{\nu}{d}})$.  For the noisy setting, the results and their tightness vary depending on the noise level.  When $\sigma = \Theta(1)$ (i.e., constant noise), a simple sampling strategy from the BO literature \citep{Vak21} leads to $O^*\big(\lambda T^{-\frac{\nu}{2\nu+d}}\big)$ regret,\footnote{Monte-Carlo algorithms achieve $O(\frac{1}{\sqrt{T}})$ when $\lambda=\Theta(1)$, which is better than the mentioned $O^*\big(T^{-\frac{\nu}{2\nu+d}}\big)$ regret, but to our knowledge, its standard analysis is unable to capture the dependence on $\lambda$ that might vary with $T$.} but the $\Omega(\lambda T^{-\frac{1}{2}})$ lower bound leaves open the possibility that a better algorithm might exist.  On the other hand, at lower noise levels there are regimes where our upper and lower bounds match, as exemplified by the case $\sigma=\Theta(T^{-\frac{1}{4}})$ (and $\nu \ge \frac{d}{2}$) with regret $\Theta(\lambda T^{-\frac{3}{4}})$.
    \item Depending on the precise scaling of $\lambda$, the error bound can vary between those of BQ and BO.  For instance, in the noiseless setting, we can control the rate of $\lambda$ going to infinity as $\lambda=O(T^{\frac{\nu}{d}})$, and obtain a $\Omega(T^{-1})$ lower bound.
\end{itemize}
The connections to BQ and BO can partially be understood with the aid of the figures depicted in Figure \ref{fig:demo}.  When $\lambda$ is very small, estimating $\int_D e^{-\lambda f(\xv)}d\xv$ is almost the same as estimating $\int_D (1-\lambda f(\xv))d\xv$, which is a shifted and scaled version of $\int_D f(\xv)d\xv$.  On the other hand, for large values of $\lambda$, the minimum value of $f(\xv)$ stands out considerably compared to the rest, as exemplified by the narrow bump in the figure.

A more detailed discussion of our contributions is deferred to the subsequent section on related work, where we provide a detailed analysis of the advancements and novel aspects of our research compared to existing works.

\section{Related Work}\label{sec:relate}
{\noindent \bf Log-concave sampling.} Over the years, a wide range of methods have been developed for estimating the normalizing constant of a probability distribution \citep{Sto10}, with a particular focus on algorithms that can leverage the log-concavity of the normalizing constant based on convex optimization \citep{Lov06,Bro18,Dwi18,Ge20}.  These algorithms have been shown to have strong theoretical guarantees when provided with access to $\nabla f(\xv)$, by building on specific and popular bounds of Langevin-based sampling algorithms \citep{Dur16}.

\vskip 0.1in
{\noindent \bf Non-log-concave Sampling.} Relatively fewer studies have explored the convergence rate for non-log-concave distributions, with the most related one being \citep{Hol23}.  Our results are generally not directly comparable with theirs due to the consideration of a different function class, leading to different algorithmic techniques (e.g., piecewise constant approximation vs.~Gaussian process methods).  In addition, one of our main goals is to handle noise, whereas \citep{Hol23} focuses on the noiseless setting.  A more detailed comparison to \citep{Hol23} can be found in Appendix~\ref{sec:comparison}.

\vskip 0.1in
{\noindent \bf Other Non-log-concave Sampling Works.} Other works that consider non-convex energy functions appear to have significantly less direct relevance to our study due to the different modeling assumptions we adopt.  However, it is worth mentioning a few of them, and interested readers can refer to the references provided in those papers for further exploration.  Some of these works specifically focus on the analysis of {\em local} non-convexity within a small region, while the functions exhibit strong convexity outside of that region \citep{Che18, Ma19}.  Additionally, there are studies that analyze upper and lower bounds related to the {\em relative Fisher information} by imposing restrictions on the smoothness of $\nabla f$ \citep{Bal22, Che23}.




\vskip 0.1in
{\noindent \bf Bayesian Optimization / GP Bandits.}  An important distinction that sets us apart from previous literature is that all of our results are rooted in the kernel-based bandit framework, which leverages Gaussian processes and Bayesian inference to approximate and optimize unknown functions.   Within this area, two prominent problems are Bayesian quadrature (BQ) \citep{Kan19, Wyn21, Cai23} for approximating integrals and Bayesian optimization (BO) \citep{Sri09, cho17, Vak21} for finding the global maximum of a black-box function. These methods have gained popularity due to their wide range of applications and favorable theoretical properties.  While BO is understood to be a more challenging problem than BQ in terms of the sample complexity, it can still be applied to provide sub-optimal BQ bounds (e.g., see \citep{Cai23}).  In our case, it can similarly be utilized to obtain (sub-optimal) bounds for NC, thanks in part to the tight noisy $L^{\infty}$ upper bound and upper bound on {\em information gain} derived by \citep{Vak21,Vak21a}.

We briefly note that various non-linear transformations of black-box functions (e.g., $e^{- \lambda f(x)}$ or $f(x)^2$) have appeared in other contexts such as optimization and regression \citep{Fla17,Ast19,Mar20}, but to our knowledge none have studied the NC problem.

\section{Problem Setup}\label{sec:setup}
Let $f\in\Hc(B)$ be an RKHS function on the compact domain $D=[0,1]^d$,\footnote{Any rectangular domain can be reduced to $[0,1]^d$ by suitable shifting and scaling.} where $\Hc(B)$ denotes the set of all functions whose RKHS norm $\|\cdot\|_{\Hc}$ is upper bounded by some constant $B>0$.  Here the RKHS norm $\|\cdot\|_{\Hc}$ depends on our choice of the kernel $k(\xv,\xv')$, and we focus our attention on the widely-adopted Mat\'ern-$\nu$ kernel: 
\begin{equation*}
    k(\xv,\xv') = \dfrac{2^{1-\nu}}{\Gamma(\nu)} \bigg(\dfrac{\sqrt{2\nu}\,\|\xv-\xv'\|}{l}\bigg)^{\nu}  J_{\nu}\bigg(\dfrac{\sqrt{2 \nu}\,\|\xv-\xv'\|}{l} \bigg), \label{eq:kMat}
\end{equation*}
where $J_{\nu}$ is the modified Bessel function, and $\Gamma$ is the gamma function.  To make the dependence on $\nu$ explicit, we use $\mat$ to denote the Mat\'ern RKHS, with its RKHS norm being represented by $\|\cdot\|_{\mat}$. Given query access to $f(\xv)$, at time step $t$, the observations $y_t$ are modeled as follows:
\begin{itemize}
    \item In the noiseless setting, we simply have $y_t = f(\xv_t)$.
    \item In the noisy setting, we have $y_t=f(\xv_t)+z_t$, where $z_t\sim \Nc(0,\sigma^2)$ is i.i.d. Gaussian noise.
\end{itemize}
Our goal is to approximate the normalizing constant,
\begin{equation} \label{eq:z}
    Z(f) =\int_{D} e^{-\lambda f(\xv)} d\xv, \quad \lambda>0,
\end{equation}
leading to an estimate $\Zhat(f)$, which we seek to be accurate within a multiplicative factor of $1\pm \epsilon$. In other words, we are interested in the quantity
\begin{equation}\label{eq:error}
    \epsilon = \sup_{f\in\mat} \Big|\frac{\Zhat(f)}{Z(f)} - 1\Big|.
\end{equation}
While the additive error (which would be $|\Zhat - Z|$ in our case) is commonly used for integration problems such as BQ, it is less suitable here unless $\lambda\to 0$ (yielding $Z \to 1$).  This is because as $\lambda$ grows large, $Z$ may approach zero (in which case $|\Zhat - Z| \le \epsilon$ may hold trivially) or infinity (in which case $|\Zhat - Z| \le \epsilon$ may be an overly stringent requirement).  The relative error \eqref{eq:error} serves to more naturally unify these various cases. 
%
%
In a similar manner, \citep{Hol23} defines the error as $\epsilon = \big|\log \Zhat - \log Z \big|$, which is essentially equivalent to \eqref{eq:error} due to the fact that $\log(1+\alpha)$ behaves as $O(\alpha)$ when $|\alpha|$ is strictly smaller than one (and as $\alpha(1+o(1))$ when $\alpha \to 0$).


\section{Lower Bounds}\label{sec:lb}

In this section, we provide algorithm-independent lower bounds on $\epsilon$ (see \eqref{eq:error}), i.e., impossibility/hardness results, for both the noiseless and noisy settings. 

\begin{thm} \label{thm:noiseless_lower}
    {\em (Noiseless Lower Bound) } Consider the noiseless setting with $f\in\mat(B)$, $\nu+\frac{d}{2}\ge 1$, and a time horizon $T\to \infty$.  For any algorithm that estimates $Z$ and produces an estimate $\hat{Z}$ satisfying \eqref{eq:error}, the worst-case $f\in\mat(B)$ must have the following lower bound on \eqref{eq:error} with $\Omega(1)$ probability:
    \begin{itemize}
        \item If $\lambda=\Theta(T^c)$ with $c\le\frac{\nu}{d}+\frac{1}{2}$, then $\epsilon=\Omega(T^{-\frac{\nu}{d}-1+c})$.
        \item If $\lambda=\Theta(\log T)$, then $\epsilon = \Omega(T^{-\frac{\nu}{d}-1} \log T)$.
    \end{itemize}
\end{thm}
The condition $c\le \frac{\nu}{d}+\frac{1}{2}$ comes from a technical condition in the proof, and we believe it to be quite mild (and similarly in other results below); in particular, we cover broad cases of interest, some of which are given as follows:
\begin{itemize}
    \item For $c=0$, we have when $\lambda=\Theta(1)$, $\epsilon=\Omega(T^{-\frac{\nu}{d}-1})$.
    \item For $c=\frac{1}{2}$, we have when $\lambda=\Theta(\sqrt{T})$, $\epsilon=\Omega(T^{-\frac{\nu}{d}-\frac{1}{2}})$.
    \item For $c=\frac{\nu}{d}$, we have when $\lambda=\Theta(T^{\frac{\nu}{d}})$, $\epsilon=\Omega(T^{-1})$.
    \item For $c=\frac{\nu}{d}+\frac{1}{2}$, when $\lambda=\Theta(T^{\frac{\nu}{d}+\frac{1}{2}})$, $\epsilon=\Omega(T^{-\frac{1}{2}})$.
    \item For $c=1$ (which requires $\nu\ge \frac{d}{2}$), we have when $\lambda=\Theta(T)$, $\epsilon=\Omega(T^{-\frac{\nu}{d}})$.
\end{itemize}

Note that the noiseless BQ lower bound is known as $\epsilon=\Omega(T^{-\frac{\nu}{d}-1})$ (see e.g. \citep{Nov06,Cai23}), and the noiseless BO lower bound is $\epsilon=\Omega(T^{-\frac{\nu}{d}})$ \citep{Bul11}.  Thus, the above results indicate that when $\lambda=\Theta(1)$ or approaches 0, the lower bound matches that of BQ, whereas when $\lambda=\Theta(T)$, the bound matches that of BO.  Although the above arguments compare additive error with relative error, this is valid since our hard function class in proving the lower bound has $\Theta(1)$ normalizing constant.  See Appendix \ref{app:class} for further details.

Next, we turn our attention to the noisy setting.

\begin{thm} \label{thm:noisy_lower}
    {\em (Noisy Lower Bound) } Consider the noisy setting with $f\in\mat(B)$, $\nu+\frac{d}{2}\ge 1$, a time horizon $T\to \infty$, and noise variance $\sigma^2$ (possibly varying with $T$).  For any algorithm that estimates $Z$ and produces an estimate $\hat{Z}$ satisfying \eqref{eq:error}, the worst-case $f\in\mat(B)$ must have the following lower bound on \eqref{eq:error} with $\Omega(1)$ probability:
    \begin{itemize}
        \item If $\lambda = \Theta(T^c)$ and $\sigma = \Theta(T^a)$ with $c \le \min\big\{\frac{\nu}{d}+\frac{1}{2}, (\frac{1}{2}-a) \frac{2\nu+d}{2\nu+2d}\big\}$ and $a\le\frac{1}{2}$, then $\epsilon = \Omega(T^{-\frac{\nu}{d}-1+c} + \sigma T^{-\frac{1}{2}+c})$.
        \item If $\lambda=\Theta(\log T)$ and $\sigma = \Theta(T^a)$ with $a<\frac{1}{2}$, then $\epsilon = \Omega(T^{-\frac{\nu}{d}-1}\log T + \sigma T^{-\frac{1}{2}}\log T)$.
    \end{itemize}
\end{thm}

Since the noisy setting involves the variables $a$ and $c$, we will focus mainly on the most studied case where $\sigma=\Theta(1)$ ($a=0$) for the sake of clarity.  For constant noise variance, the BQ noisy lower bound is known to be $\Omega(T^{-\frac{1}{2}})$ \citep{Pla96, Cai23}, and the BO noisy lower bound is known as $\epsilon=\Omega(T^{-\frac{\nu}{2\nu+d}})$ \citep{Sca17,Cai21}.  To compare against these, for the first dot point in Theorem \ref{thm:noisy_lower}, we consider the following specific $c$ values:
\begin{itemize}
    \item For $c=0$, we have when $\lambda=\Theta(1)$, $\epsilon=\Omega(T^{-\frac{1}{2}})$, which coincides with the noisy BQ lower bound.
    \item For $c=\frac{d}{4\nu+2d}$ (needs $\nu\ge \frac{\sqrt{5}-1}{4}d$, so that $\frac{d}{4\nu+2d} \le \frac{2\nu+d}{4\nu+4d}$), we have when $\lambda=\Theta(T^{\frac{d}{4\nu+2d}})$, $\epsilon=\Omega(T^{-\frac{\nu}{2\nu+d}})$, which coincides with the noisy BO lower bound.  Note that this particular choice of $\lambda$ is somewhat artificial, but it highlights the fact that our result can lead to BO-like results.
    \item For $c=\frac{1}{3}$ (needs $\nu\ge \frac{d}{2}$), we have when $\lambda=\Theta(T^{\frac{1}{3}})$, $\epsilon=\Omega(T^{-\frac{1}{6}})$.
    \item For $c=\frac{1}{4}$, we have when $\lambda=\Theta(T^{\frac{1}{4}})$, $\epsilon=\Omega(T^{-\frac{1}{4}})$.
    \item For $c=\frac{2\nu+d}{4\nu+4d}$ (the maximum allowed value), we have when $\lambda=\Theta(T^{\frac{2\nu+d}{4\nu+4d}})$, $\epsilon=\Omega(T^{ -\frac{d}{4\nu+4d}})$.
\end{itemize}

Further analysis of this lower bound can be found in Appendix \ref{app:noisy_lb}; similar to the above results, we may end at BQ or BO type of error bounds for different $a$ and $c$ values.  To illustrate an example, we recall the result shown in Table \ref{tab:results} when $a=-\frac{1}{2}$, where the lower bound is $\Omega(\lambda T^{-\frac{1}{2}})$. Substituting $a=-\frac{1}{2}$ into Theorem \ref{thm:noisy_lower}, it can be seen that the value of $c$ lies in the range $(-\infty, \frac{2\nu+d}{2\nu+2d}]$.  By choosing $\lambda=\Theta(T^{-\frac{\nu}{d}})$, we obtain the BQ-like lower bound $\Omega = \Omega(T^{-\frac{\nu}{d}-\frac{1}{2}})$, whereas by choosing $\lambda=\Theta(T^{\frac{2\nu+d}{2\nu+2d}})$, we obtain the BO-like lower bound $\Omega = \Omega(T^{-\frac{d}{4\nu+4d}})$.

In the following section, we will derive algorithmic upper bounds that sometimes match the algorithm-indpeendent lower bounds, though with gaps remaining in other cases.

\section{Upper Bounds}\label{sec:upper}
\begin{algorithm}[t]
	\caption{Two-batch normalizing constant estimation algorithm}  \label{alg:two-phase-nc}
	\begin{algorithmic}[1]
		\STATE {\bfseries Input:} Domain $D$, $\GP(0,k)$ prior, GP hyperparameter $\xi$, time horizon $T$, noise standard deviation $\sigma$.
		\FOR{$t=1,\dots,T/2$}
			\STATE Select $\xv_t=\argmax_{\xv\in D}\sigma_{t-1}(\xv)$.
			\STATE Update $\sigma_t$ using $\xv_1,\dotsc,\xv_t$.
		\ENDFOR
		\STATE Update $\mu_{T/2}(\xv)$ using $\{\xv_t\}_{t=1}^{T/2}$ and $\{y_t\}_{t=1}^{T/2}$.
		\FOR{$t=T/2+1,\dots,T$}
			\STATE Sample $\xv_t \sim \frac{e^{-\lambda\mu_{T/2}(\xv)}}{\int_D e^{-\lambda\mu_{T/2}(\xv)}d\xv}$ independently. \label{line:samp}
			\STATE Receive $y_t=f(\xv_t) + z_t$
		\ENDFOR
		\STATE Compute the approximate integral $\Zhat_1 =\int_D e^{-\lambda\mu_{T/2}(\xv)}d\xv$
		\STATE Compute the residual $\Rhat=\frac{2}{Te^{\lambda^2\sigma^2/2}}\sum_{t=T/2 + 1}^{T} e^{\lambda\mu_{T/2}(\xv_t)-\lambda y_t}$
		\STATE Output $\Zhat = \Zhat_1 \cdot \Rhat$
	\end{algorithmic}
\end{algorithm}

We present a GP-based two-batch algorithm for estimating the normalizing constant in Algorithm \ref{alg:two-phase-nc}. 
 Given a GP prior model $\GP(0,k)$, after observing $t$ samples, the posterior distribution is also a GP with the following posterior mean and variance:
\begin{align}
    \mu_{t}(\xv) &= \kv_t(\xv)^T\big(\Kv_t + \xi \mathbf{I}_t \big)^{-1} \yv_t,  \label{eq:posterior_mean} \\ 
    \sigma_{t}^2(\xv) &= k(\xv,\xv) - \kv_t(\xv)^T \big(\Kv_t + \xi \mathbf{I}_t \big)^{-1} \kv_t(\xv), \label{eq:posterior_variance}
\end{align}
where $\yv_t = [y_1,\ldots,y_t]^T$, $\kv_t(\xv) = \big[k(\xv_i,\xv)\big]_{i=1}^t$, $\Kv_t = \big[k(\xv_t,\xv_{t'})\big]_{t,t'}$ is the kernel matrix, $\mathbf{I}_t$ is the identity matrix of dimension $t$, and $\xi>0$ is a hyperparameter. 
\begin{figure*}[!t]
    \centering
    \begin{subfigure}{\linewidth}
        \centering
        \includegraphics[width=\sizecustom]{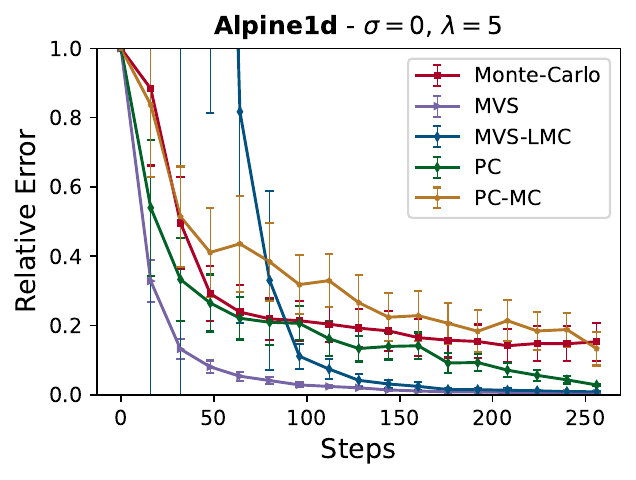}
        \includegraphics[width=\sizecustom]{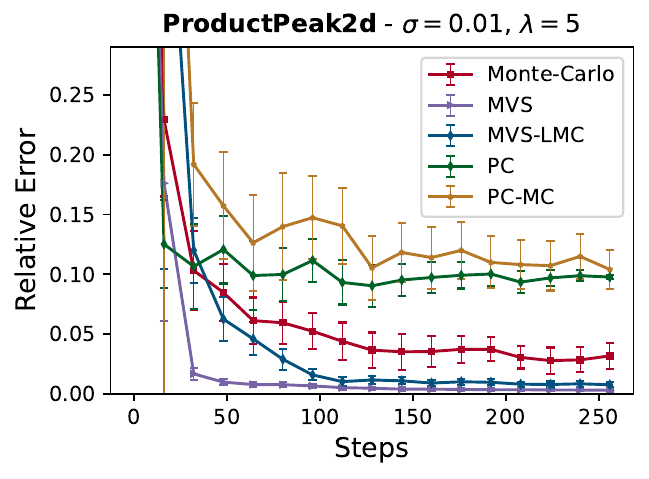}
        \includegraphics[width=\sizecustom]{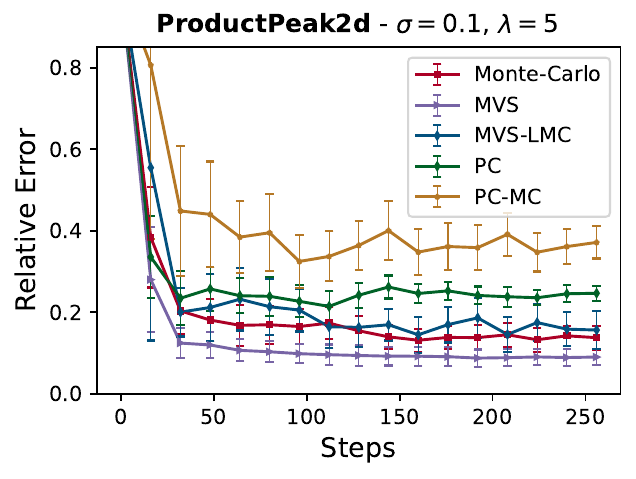}
    \end{subfigure}
    \begin{subfigure}{\linewidth}
        \centering
        \includegraphics[width=\sizecustom]{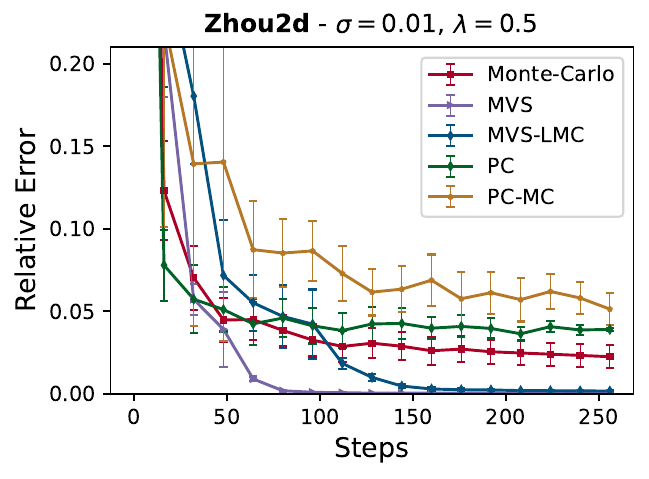}
        \includegraphics[width=\sizecustom]{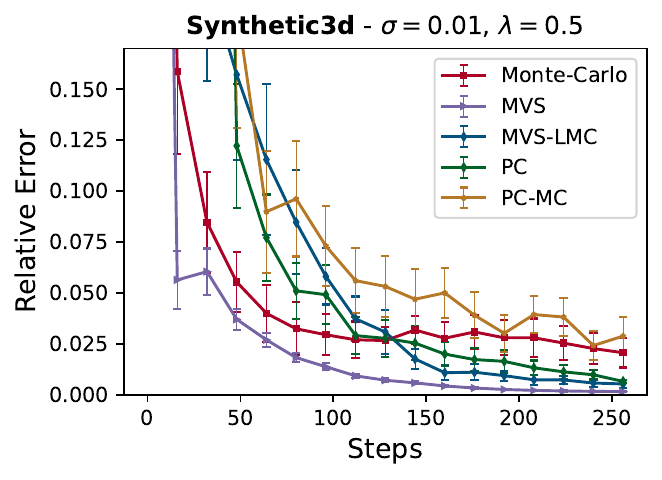}
        \includegraphics[width=\sizecustom]{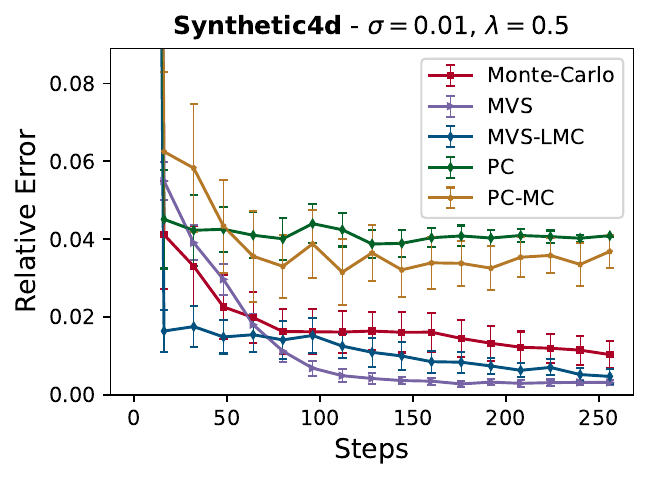}
    \end{subfigure}
    
    \caption{Results for analytic functions. \label{fig:analytic}}
\end{figure*}

In Algorithm \ref{alg:two-phase-nc}, we initially employ $\frac{T}{2}$ samples to construct a GP approximation of $f$ in a non-adaptive manner (i.e., the selection of $\xv_1,\ldots,\xv_{T/2}$ occurs prior to observing $y_1,\ldots,y_{T/2}$).  This gives us an estimate $\mu_{T/2}(\cdot)$ of the entire function, from which we can form an initial estimate $\hat{Z}_1$ of $Z$.  We then refine the estimate by using Monte Carlo sampling to estimate a multiplicative ``residual'' term (with estimate $\hat{R}$), and the final estimate $\hat{Z}$ is the product $\hat{Z}_1 \cdot \hat{R}$. 

The two-batch idea is most easily understood in the BQ problem of integrating $f$: Roughly speaking, the error of Monte Carlo decays as $\frac{1}{\sqrt{T}}$ but is also proportional to the function scale itself, so by applying it to the residual, we get the original error times $\frac{1}{\sqrt{T}}$.  The intuition in our NC problem is generally similar, but the non-linearity of $e^{-\lambda f}$ leads us to consider a multiplicative residual, and the analysis becomes more complicated.

As discussed in Section \ref{sec:relate}, \citep{Hol23} has utilized a similar algorithm by selecting grid points and forming a piecewise constant approximation in the first batch, which they find to be optimal for a different once-differentiable function class, but is less suitable for our setting (see Appendix \ref{sec:comparison} and Section \ref{sec:exp}). Simpler variants of this idea have also been used in BQ; see \citep{Cai23} and the references therein.

We summarize the error bounds of Algorithm \ref{alg:two-phase-nc} (as well as the estimate resulting from the first batch alone) in the following two theorems.
\begin{thm} \label{thm:noiseless_upper}
    {\em (Noiseless Upper Bound)} Consider our problem setup with constant parameters $(B,\nu,d,l)$, and time horizon $T\to \infty$.  With probability at least $1-\delta$ (for an arbitrary fixed $\delta \in (0,1)$),
    the relative error incurred by Algorithm \ref{alg:two-phase-nc} with $\xi=0$ has the following upper bounds on \eqref{eq:error}:
    \begin{itemize}
        \item If $\lambda = \Theta(T^{c})$ with $c\le \frac{\nu}{d}$, then $\epsilon = O\big(T^{-\frac{\nu}{d}-\frac{1}{2}+c} \big)$.
        \item If $\lambda = \Theta(\log T)$, then $\epsilon= O\big( T^{-\frac{\nu}{d} -\frac{1}{2}}\log T\big)$.
    \end{itemize}
\end{thm}

Comparing to the noiseless lower bound obtained in Theorem \ref{thm:noiseless_lower}, there is a non-negligible gap of $O(\sqrt{T})$.  However, this difference is fairly insignificant when $\nu\gg d$.

\begin{thm} \label{thm:noisy_upper}
    {\em (Noisy Upper Bound)} Consider our problem setup with constant parameters $(B,\nu,d,l,\xi)$, noise standard deviation $\sigma$ that may scale with $T$ (i.e., $\sigma=\Theta(T^a)$), and time horizon $T\to \infty$. Then, the following upper bounds on \eqref{eq:error} hold with probability at least $1-\delta$ for any constant $\delta > 0$:
    \begin{itemize}
        \item If $\lambda=\Theta(T^{c})$ with $a+c \le 0$ and $c<\frac{\nu}{2\nu+d}$, Algorithm \ref{alg:two-phase-nc} yields $\epsilon = O\big(T^{-\frac{\nu}{2\nu+d}-\frac{1}{2}+c}(\log T)^{\frac{\nu}{2\nu+d}} + \sigma T^{-\frac{1}{2}+c}\big)$.
        \item If $\lambda=\Theta(\log T)$ with $a\le 0$, Algorithm \ref{alg:two-phase-nc}  yields $\epsilon = O\big(T^{-\frac{\nu}{2\nu+d}-\frac{1}{2}}(\log T)^{\frac{3\nu+d}{2\nu+d}} + \sigma T^{-\frac{1}{2}}\log T\big)$.
    \end{itemize}
    In addition, the following upper bounds hold with probability at least $1-\frac{1}{T^{\alpha}}$ with $\alpha$ being any fixed constant:
    \begin{itemize}
        \item If $\lambda=\Theta(T^{c})$ with $a+c<\frac{\nu}{2\nu+d}$ and $c<\frac{\nu}{2\nu+d}$, \underline{the intermediate estimate $\Zhat_1$} from Algorithm \ref{alg:two-phase-nc} yields $\epsilon=O\big(T^{-\frac{\nu}{2\nu+d}+c}(\log T)^{\frac{\nu}{2\nu+d}} + \sigma T^{-\frac{\nu}{2\nu+d}+c}(\log T)^{\frac{4\nu+d}{4\nu+2d}}\big)$.
        \item If $\lambda=\Theta(\log T)$ with $a<\frac{\nu}{2\nu+d}$, \underline{the intermediate estimate $\Zhat_1$} from Algorithm \ref{alg:two-phase-nc} yields $\epsilon=O\big(T^{-\frac{\nu}{2\nu+d}}(\log T)^{\frac{3\nu+d}{2\nu+d}} + \sigma T^{-\frac{\nu}{2\nu+d}}(\log T)^{\frac{8\nu+3d}{4\nu+2d}}\big)$.
    \end{itemize}
\end{thm}

The second part of the theorem helps to broaden the range of allowed $(a,c)$ pairs.  However, when $(a,c)$ is feasible for the first part, it gives a stronger result than the second part, improving by a $O(\frac{1}{\sqrt{T}})$ factor in the non-$\sigma$ term and by $O(\frac{d}{4\nu+2d})$ in the $\sigma$ term.  This highlights the benefit of using both batches, at least in terms of upper bounds.

Comparing to the noisy lower bound obtained in Theorem \ref{thm:noisy_lower}, the convergence rate obtained in Theorem \ref{thm:noisy_upper} is less straightforward, but we give some analysis as follows:
\begin{itemize}
    \item For the first dot points in Theorem \ref{thm:noisy_lower} and Theorem \ref{thm:noisy_upper}, the two results align under the high noise regime when $a\ge -\frac{\nu}{2\nu+d}$ (this threshold on $a$ is obtained by equating the exponents $-\frac{\nu}{2\nu+d}-\frac{1}{2}+c$ and $a-\frac{1}{2}+c$ from Theorem \ref{thm:noisy_upper}).  In this regime, the order of the error is optimal as $\Theta(\sigma T^{-\frac{1}{2}+c})$.  This matches the order-optimal bound for BQ (upon replacing $c$ by zero, since $\lambda$ is absent in BQ).  As a specific example, as shown in Table \ref{tab:results}, when we assume $a=-\frac{1}{4}$ and $\nu\ge\frac{d}{2}$ (which ensures $-\frac{1}{4}\ge -\frac{\nu}{2\nu+d}$), the resulting upper bound is optimal at $\Theta(T^{-\frac{3}{4}+c})$.
    \item For an extreme low noise regime ($a \le -\frac{\nu}{d}-\frac{1}{2}$, which ensures that the first term in Theorem \ref{thm:noisy_lower} dominates the second term), the lower bound vs.~the upper bound is $\Omega(T^{-\frac{\nu}{d}-1+c})$ vs.~$O(T^{-\frac{\nu}{2\nu+d}-\frac{1}{2}+c})$.  Thus, the relative gap is $O(T^{\frac{d}{4\nu+2d}+\frac{\nu}{d}})$, which is around $O(\sqrt{T})$ if $d\gg \nu$.
    \item Otherwise, if $ -\frac{\nu}{d}-\frac{1}{2}<a<-\frac{\nu}{2\nu+d}$, the lower bound vs. the upper bound is $\Omega(\sigma T^{-\frac{1}{2}+c})$ vs. $O(T^{-\frac{\nu}{2\nu+d}-\frac{1}{2}+c})$.  The relative gap is now $O(T^{-\frac{\nu}{2\nu+d}-a})$, which primarily depends on the value of $a$ if $d \gg \nu$.  In this sense, the gap ranges from $O(T^{o(1)})$ to $O(T)$.  To illustrate an example under this condition, see Table \ref{tab:results} with the choice $a=-\frac{1}{2}$.
    \item The last two bullet points in the theorem address the remaining scenario when $\lambda\sigma\to \infty$ (i.e., $a+c>0$), and the term $\frac{\nu}{2\nu + d}$ in the exponent matches that observed for simple regret in the BO literature \citep{Sca17,Vak21}, though failing to match the upper bound for NC.  Hence, the aforementioned noisy upper bounds demonstrate that, albeit with some gaps, when $\lambda\sigma\to 0$, the derived upper bound shares similarities with BQ, whereas when $\lambda\sigma\to \infty$, the bound shares similarities with BO.
\end{itemize}

\begin{figure*}[!t]
\centering
    \begin{subfigure}{\linewidth}
        \centering
        \includegraphics[width=\sizecustom]{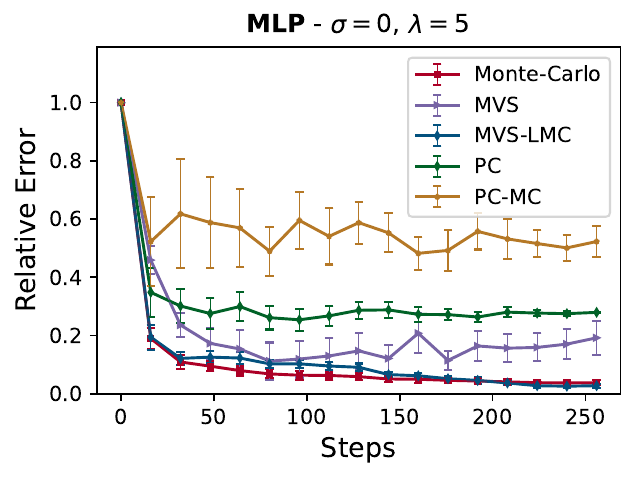}
        \includegraphics[width=\sizecustom]{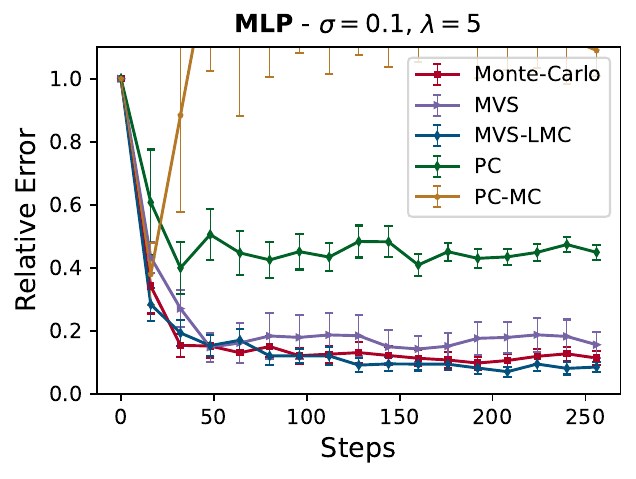}
        \includegraphics[width=\sizecustom]{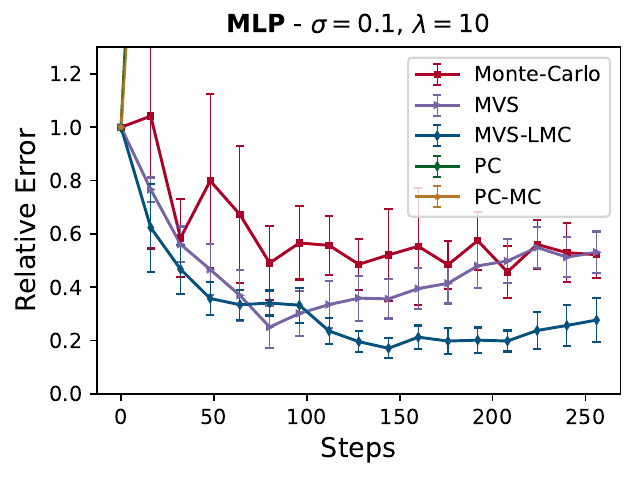}
    \end{subfigure}
    \caption{Results for MLP. \label{fig:mlp}}
\end{figure*}

\section{Experiments} \label{sec:exp}
In this section, we conduct simulation studies to investigate Algorithm \ref{alg:two-phase-nc}  and its intermediate estimate $\Zhat_1$. 

\subsection{Setup}
{\bf \noindent Sampling in the Second Batch.} As shown in line 8 of Algorithm \ref{alg:two-phase-nc}, the target distribution is proportional to $e^{-\lambda\mu_{T/2}(\xv)}$, which can be challenging to sample exactly.  Fortunately, there is a vast literature on approximate sampling methods that we can use.  We choose to use the method of Langevin dynamics defined by the following stochastic differential equation (SDE) \citep{Uhl30}:
\begin{equation}\label{eq:sde}
    d \Xv(t) = -\nabla g(\Xv(t)) dt + \sqrt{2\lambda^{-1}} \Bv(t),
\end{equation}
where $\lambda>0$ is interpreted as the inverse temperature, and $\Bv(t)\in \RR^d$ is the Brownian motion at time $t$.  
A standard approach to solve \eqref{eq:sde} is to apply Euler-Maruyama discretization, leading to the following {\em Langevin Monte-Carlo} (LMC) updating rule:
\begin{equation}\label{eq:lmc}
    \xv_{t+1} = \xv_{t} - \beta\nabla \mu_{T/2}(\xv_t) + \sqrt{2\beta\lambda^{-1}} \beps_t,
\end{equation}
where we have replaced $g$ with $\mu_{T/2}$, $\beps_t$ are i.i.d. standard Gaussian random vectors in $\RR^d$, and $\beta>0$ is the step size of the discretization.  Note that using this sampling strategy incurs some approximation error that we do not attempt to account for in our theory (analogous to how BO theory assumes exact acquisition function optimization).

\vskip 0.1in
{\bf \noindent Hyperparameters.} For all functions considered in this section, we consider a time horizon of $T = 256$, $\lambda\in\{0.5, 5, 10\}$, $\sigma\in\{0, 0.01, 0.1\}$ and $\nu\in\{0.5,1.5,2.5\}$.  The total number of steps of \eqref{eq:lmc} is set as $20$, and the LMC learning rate is $\beta = 10^{-3}$.  We adopt two learnable kernel hyperparameters in \eqref{eq:kMat}, the lengthscale $l$ and an additional scale parameter (multiplying $k$), to permit functions with varying ranges (while $\nu$ remains fixed).  Except for synthetic functions where the true hyperparameters are known, these two parameters are optimized by maximizing the data log-likelihood \citep{Ras06} using the built-in SciPy optimizer based on L-BFGS-B, which is also used for finding the maximum variance point in Algorithm \ref{alg:two-phase-nc}.

\vskip 0.1in
{\bf \noindent Benchmarks.}  In addition to the commonly adopted Monte-Carlo quadrature baseline, as discussed in Sections \ref{sec:relate} and \ref{sec:upper}, the most closely related work by  \citep[Sec.~5.1]{Hol23} proposes the use of piecewise constant approximation to estimate NC with grid inputs, which also achieves improved theoretical convergence when combined with an additional MC step.  We adopt their shorthand notations and refer to these two benchmarks as PC and PC-MC, respectively.

\vskip 0.1in
{\bf \noindent Evaluation.}  We refer to the first batch of Algorithm \ref{alg:two-phase-nc} as maximum variance sampling (MVS),\footnote{More precisely, MVS corresponds to taking all $T$ samples based on the maximum variance rule, not just the first $T/2$.} and the whole Algorithm \ref{alg:two-phase-nc} as MVS-LMC.  We evaluate the performance using the mean absolute relative error, with the ground truth value (and also $\Zhat_1$ at Line 11 of Algorithm \ref{alg:two-phase-nc}) being determined by trapezoidal rule with $10^5$ uniformly-spaced grid points (without noise).  Error bars in our plots indicate $\pm 0.5$ standard deviation with respect to the 100 trials.

\subsection{Analytic Functions}\label{sec:exp_analytic}
In order to assess the empirical behaviour of MVS and MVS-LMC, we first conduct experiments on the following analytic functions for $d\in\{1,2,3,4\}$:

{\bf  Synthetic functions.}  The synthetic functions are constructed by sampling $m=30d$ points, $\xvhat_1\ldots\xvhat_m$, uniformly on $[0,1]^d$, and $\ahat_1\ldots\ahat_m$ uniformly on $[-1,1]$.  The function is then defined as $f(\xv) = \sum_{i=1}^{m} \ahat_i k(\xvhat_i, \xv)$.  The length-scale and $\nu$ are set to be fixed (no hyperparameter learning) as $0.2$ and $2.5$ respectively.

{\bf  Benchmark functions.}  Exact formulations of functions including, Ackley, Alpine, Product-Peak, Zhou, etc., can be found in \citep{SFU_Funcs}.  $\nu$ is fixed as $3/2$ for all of these benchmark functions.

\subsection{Multi-layer Perception (MLP)}
For a more complex scenario, we consider an 8-dimensional MLP function, with the structure being defined by
\begin{equation*}
    f(\xv) = f^4(\Tanh(f^3(\Tanh(f^2(\Tanh(f^1(\xv))))))),
\end{equation*}
where $\Tanh(\cdot): \RR^n\to [-1,1]^n$ is the Hyperbolic tangent (tanh) activation function, and the dimension mapping from layers $f^1$ to $f^4$ is $8\to16\to32\to16\to1$.  We use Xavier initialization to set and fix the weights of the MLP, and set $\nu = 1/2$ to model the potentially more erratic behavior.

\ifthenelse{\boolean{ARXIV}}
{
\begin{figure}[!t]
    \centering
    \begin{subfigure}[t]{\sizecustom}
        \centering
        \includegraphics[width=\columnwidth]{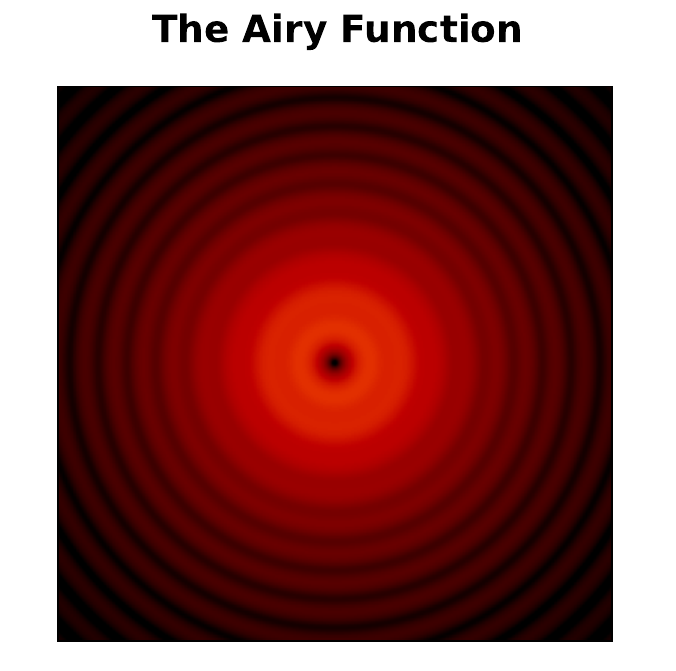}
        \caption{PSF computed for the wavelength $2\times 10^{-6}$.}
        \label{fig:airy}
    \end{subfigure}
    \begin{subfigure}[t]{\sizecustom}
        \centering
        \includegraphics[width=\columnwidth]{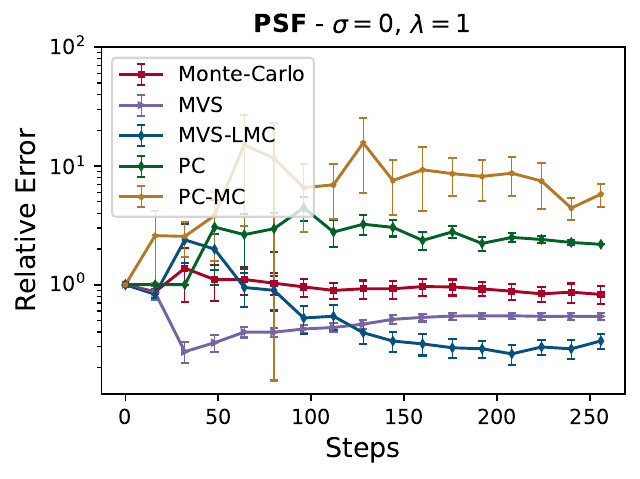}
        \caption{PSF error plot.}
        \label{fig:psf_error}
    \end{subfigure}
    \begin{subfigure}[t]{\sizecustom}
        \centering
        \includegraphics[width=\columnwidth]{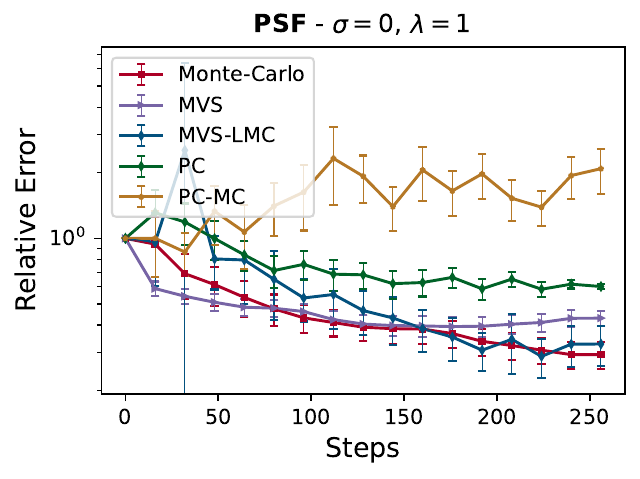}
        \caption{PSF error plot by shifting $(+0.05,+0.05)$, so that the new optimum is at $(- 0.05, - 0.05)$.}
        \label{fig:psf_error_shift}
    \end{subfigure}
    
    \caption{Results for estimating PSF. \label{fig:psf}}
\end{figure}
}
{
\begin{figure}[!t]
    \centering
    \begin{subfigure}[t]{\columnwidth}
        \centering
        \includegraphics[width=0.45\columnwidth]{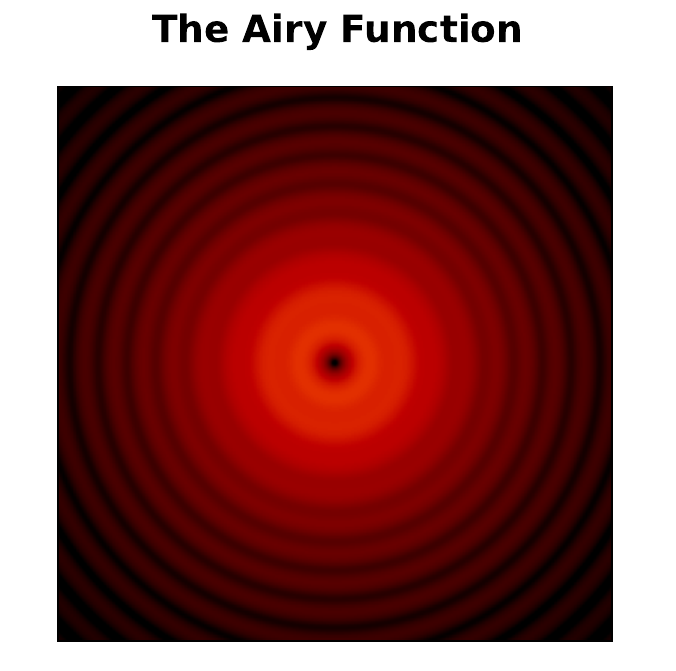}
        \caption{PSF computed for the wavelength $2\times 10^{-6}$.}
        \label{fig:airy}
    \end{subfigure}
    \vskip 0.05in
    \begin{subfigure}[t]{\columnwidth}
        \centering
        \includegraphics[width=0.6\columnwidth]{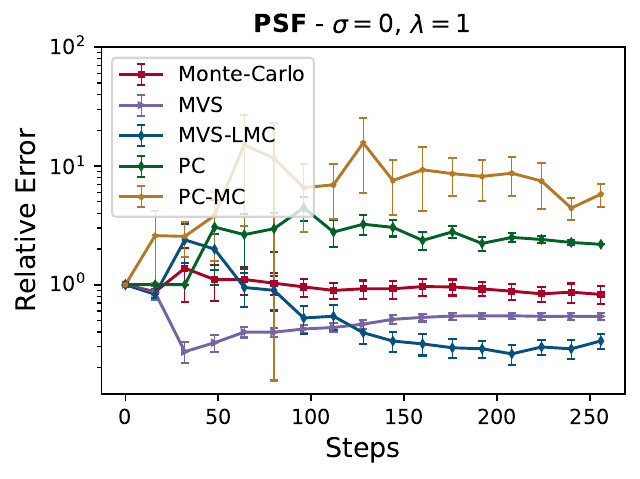}
        \caption{PSF error plot.}
        \label{fig:psf_error}
    \end{subfigure}
    \begin{subfigure}[t]{\columnwidth}
        \centering
        \includegraphics[width=0.6\columnwidth]{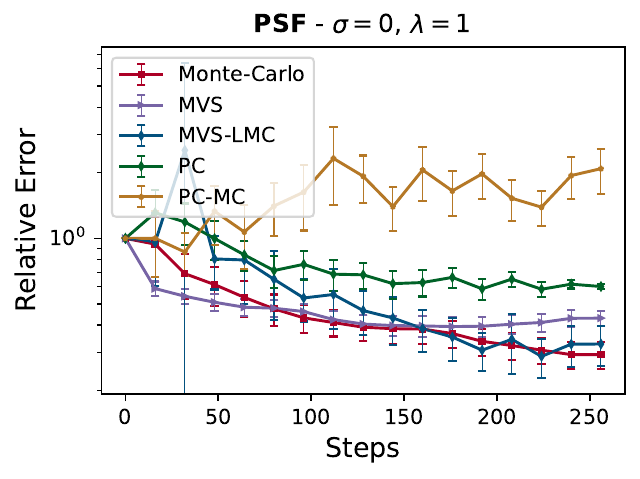}
        \caption{PSF error plot by shifting $(+0.05,+0.05)$, so that the new optimum is at $(- 0.05, - 0.05)$.}
        \label{fig:psf_error_shift}
    \end{subfigure}
    
    \caption{Result for estimating PSF. \label{fig:psf}} \vspace*{-2ex}
\end{figure}
}
\subsection{Point Spread Function (PSF)}

Beyond functions with analytic forms, we have also simulate on a diffraction energy distribution characterised by an intensity of wave-field (i.e., PSF).  This idea leads to an interesting class of functions for black-box problems, e.g., 
it has been previously evaluated in BQ tasks \citep{Nas21}.  The PSF will be dependent on the shape of the pupil (circle, rectangle, etc.), and on the wavelength of the used light.  In the case of a circular pupil, the diffracted wave pattern is known as the Airy pattern, where in our case, the logarithm of the energy intensity is regarded as a 2-D black-box energy function.  The Airy pattern generated by light with wavelength $2\times 10^{-6}$ is displayed in Figure \ref{fig:airy}, and we perform NC estimation within a quarter of the support, namely $[0, 0.5] \times [0, 0.5]$ (top-right corner of Figure  \ref{fig:airy}). Additionally, Figure \ref{fig:psf_error} displays the error curves (plotted with $\nu=0.5$, $\sigma=0$, and $\lambda=1$) using a logarithmic scale to capture the behavior of all the algorithms more effectively.

\subsection{Discussion}
As can be observed from the displayed figures, the performance of MVS and MVS-LMC outperforms MC, PC, and PC-MC by varying margins.  We find that despite the PC approach working well for higher query budgets \citep{Hol23} (e.g., $10^5$ or more), GP-based methods tend to be superior at low budgets (e.g., a few hundred).  This is in alignment with the strong query complexity properties of GP methods observed in other tasks (e.g., regression, optimization, etc.).

In particular, the number of dimensions can significantly impact the performance of PC under a small budget, as observations become more sparse in higher dimensions.  For example, in the Alpine1d plot, PC converges quickly, while it converges slowly for Synthetic4d.  On the other hand, in the PSF experiment, PC suffers from a relative error around 3 due to consistently sampling from the $(0,0)$ point, which has a significantly higher value than other locations. As seen in Figure \ref{fig:psf_error_shift}, this can be improved by shifting $f$ so that the peak is away from $(0,0)$, but only to a limited extent.


When comparing MVS and MVS-LMC to MC on simple analytic functions (see Figure \ref{fig:analytic}), the empirical behavior is consistent with our theory; that is, for a small inverse temperature $\lambda$ and noise variance $\sigma^2$, both MVS and MVS-LMC exhibit a substantial improvement over MC.  However, for larger values of $\lambda$ and $\sigma$, the error of MVS-LMC essentially reduces to $O(T^{-\frac{1}{2}})$, the same order as MC.

We observe that MVS-LMC works particularly well for the MLP and PSF functions, whose results are shown in Figures \ref{fig:mlp} and \ref{fig:psf_error}.  While the total samples are split half-by-half in Algorithm \ref{alg:two-phase-nc} for simplicity, the performance of MVS-LMC could be further improved in practice by choosing a problem-dependent split size, similarly to BQ in \citep{Cai23}.  We also note that MVS usually also works well, indicating that the LMC component is not always necessary, but the LMC component clearly helps in some cases (specifically, for MLP and PSF).

\subsection{Effect of Varying $\lambda$}

Our theoretical results indicate that the complexity of NC increases with higher values of $\lambda$.  This trend is also observed empirically, as demonstrated by the MLP plots in Figure \ref{fig:mlp}.  See also Appendix \ref{app:add_exp} for additional results of this kind.

\section{Conclusion}

Our work contributes to the understanding of the estimation of the normalizing constant for functions in an RKHS, and provides insights into the relationship between the error bound, the problem parameter $\lambda$, and the noise variance $\sigma^2$.  In general it is still an open question to what extent our bounds can be improved.  Our upper bounds on the convergence were primarily established using $L^{\infty}$ bounds in BO, which leads to an extra $\Theta(\sqrt{T})$ factor compared to our lower bounds. Improvements may be possible if we can instead build on $L^2$ function approximation bounds.  It would also be of interest to better understand the {\em squared exponential} (SE) kernel, for which some of our techniques become infeasible (e.g., the use of disjoint bump functions in the lower bound).


\section*{Acknowledgments} This work was supported by the Singapore National Research Foundation (NRF) under grant number A-0008064-00-00.

\bibliography{ref}

\newpage
\onecolumn 

{\huge \bf \centering Appendix \par}
\bigskip
{\large \bf \centering \bf Kernelized Normalizing Constant Estimation: \\ Bridging  Bayesian Quadrature and Bayesian Optimization (AAAI 2024) \par}

\medskip
{\large \centering Xu Cai and Jonathan Scarlett \par}
\medskip

\appendix
\section{Proofs of Lower Bounds}

\subsection{Function Class Construction} \label{app:class}
Let $h(\xv) = \exp\big(\frac{-1}{1-\|\xv\|^2}\big)$ for $\|\xv\|\le 1$ and $h(\xv)=0$ for $\|\xv\|> 1$,  and let $g(\xv)=\frac{-\eta}{h(\bzero)}h(\frac{2\xv}{w})$ be the scaled and reflected version of $h(\xv)$, with width $w$ and height $\eta$.  We can fill the domain $D$ with
\begin{equation}\label{eq:M_w}
    M= \Big\lfloor \frac{1}{w}\Big\rfloor^d
\end{equation}
disjoint functions $\{g_i\}_{i=1}^{M}$, where each $g_i(\xv)$ is the shifted version of $g(\xv)$.  Then, we can define a function class containing $2^M$ possible functions of the form $f(\xv)=\sum_{i=1}^{M}S_i g_i(\xv)$, where $S_i\in\{0,1\}$.  We consider the uniform distribution over this class, which is equivalent to each $S_i$ being an independent Bernoulli random variable with probability $1/2$.  According to \citep{Bul11,Cai21}, the RKHS norm of a single bump $g(\xv)$ is bounded by 
\begin{equation}
    \|g(\xv)\|_{\mat} \le O\Big(\frac{\eta}{w^{\nu}}\Big). \label{eq:g_rkhs}
\end{equation}
More generally, for a function with up to $M$ disjoint bumps (as we consider), we can bound the overall RKHS norm of $f$ as follows under the mild condition $\nu+\frac{d}{2}\ge1$ \citep[App. C.1]{Cai23}:
\begin{align}
    \|f(\xv)\|_{\mat} \le \sqrt{M} \|g\|_{\mat}.
\end{align}
Combining \eqref{eq:M_w} and \eqref{eq:g_rkhs}, it follows that $\|f(\xv)\|_{\mat}\le B$ with a choice of $\eta$ satisfying
\begin{equation}
    \eta = \Theta\Big(\frac{B}{M^{\frac{\nu}{d}+\frac{1}{2}}}\Big). \label{eq:eta_M}
\end{equation}

Given the disjointedness of the regions, we use $Z_i = \int_{D_i} e^{-\lambda S_i g_i(\xv)}d\xv$ to denote the normalizing constant of $f$ in the $i$-th region $D_i$.  We refer to region $D_i$ as being of \emph{type 0} if $S_i=0$, and otherwise we refer to $D_i$ as being of \emph{type 1}.   The values of $Z_i$ for type 0 regions are simply $w^d$, since the function value is 0.  By denoting the integration gap between type 0 and type 1 regions as $\Delta Z$, for any $f$ in the constructed hard function class, the overall normalizing constant can be obtained as
\begin{equation}
    Z(f) = \int_D e^{-\lambda f(\xv)} d\xv = 1 + \Delta Z  \sum_{i=1}^M S_i, \label{eq:nc_value}
\end{equation}
due to the fact that function values in type 1 regions are translated copies of $g(\xv)$ and the regions are disjoint.  Furthermore, we can lower bound $\Delta Z$ as follows:
\begin{align}
    \Delta Z &= \int_{[-\frac{w}{2},\ldots,-\frac{w}{2}]}^{[\frac{w}{2},\ldots,\frac{w}{2}]} \big(e^{-\lambda g(\xv)} -1\big) d\xv \nonumber\\
    &\ge 2^d \int_{\bzero}^{[\frac{w}{4},\ldots,\frac{w}{4}]} \big(e^{-\lambda g(\xv)} -1\big) d\xv \nonumber\\
    &\ge 2^d \int_{\bzero}^{[\frac{w}{4},\ldots,\frac{w}{4}]} \big(e^{-\lambda g([\frac{w}{4},\ldots,\frac{w}{4}])} -1\big) d \xv \nonumber\\
    &= \frac{w^d}{2^d} \big(e^{\lambda \eta \frac{h([\frac{1}{2},\ldots,\frac{1}{2}])}{h(\bzero)}} -1 \big) \nonumber\\
    &= c_1 w^d (e^{c_2 \lambda \eta} - 1), \quad c_1:=\frac{1}{2^d}, \quad c_2:=\frac{h([\frac{1}{2},\ldots,\frac{1}{2}])}{h(\bzero)}, \label{eq:deltaz_lb}
\end{align}
where the first three steps use the symmetry of $g(\xv)$, and the fact that $-\lambda g(\xv)$ is monotonically decreasing with respect to $\|\xv\|\in(0,1)$.  The second last step is a result of integrating a constant over a rectangular domain.  

Similarly, the upper bound of $\Delta Z$ can be derived as
\begin{align}
    \Delta Z &= \int_{[-\frac{w}{2},\ldots,-\frac{w}{2}]}^{[\frac{w}{2},\ldots,\frac{w}{2}]} \big(e^{-\lambda g(\xv)} -1\big) d\xv \nonumber\\
    &\le 2^d \int_{\bzero}^{[\frac{w}{2},\ldots,\frac{w}{2}]} \big(e^{-\lambda g(\bzero)} -1\big) d\xv \nonumber\\
    &= w^d (e^{\lambda \eta} - 1). \label{eq:deltaz_ub}
\end{align}

By substituting the above upper bound into \eqref{eq:nc_value} and recalling \eqref{eq:M_w}, we obtain an upper bound on the overall normalizing constant, which is given by 
\begin{equation}\label{eq:nc_ub}
    Z(f) \le O\big(e^{\lambda \eta}\big).
\end{equation}

\subsection{Unified Proof of Theorem \ref{thm:noiseless_lower} and Theorem \ref{thm:noisy_lower} (Noiseless and Noisy Lower Bounds)} \label{app:lb}

As hinted in Section \ref{sec:setup}, estimating the relative error defined in \eqref{eq:error} is equivalent to estimating the mean absolute error of \eqref{eq:nc_value} when $Z(f)=\Theta(1)$.\footnote{This can be seen by considering $\epsilon_1 = |\xhat-x|$ and $\epsilon_2 = |\frac{\xhat-x}{x}|$; then $\epsilon_1 = x\epsilon_2$, which scales as $\Theta(\epsilon_2)$ if $x=\Theta(1)$.}  Consequently, when deriving the lower bound, our focus lies on the scenario where $\lambda\eta$ is bounded (typically approaching zero) in \eqref{eq:nc_ub}, while still allowing for cases that $\lambda\to\infty$.  Broadly speaking, we follow a similar idea of the lower bound proofs given in \citep{Cai23}, reducing the problem of estimating $Z$ to estimating the sum of independent $\{0,1\}$ valued random variables. 

\subsubsection{Relating Estimation Error to Posterior Variance}

Conditioned on observed samples, we are interested in characterizing the posterior distribution of $\sum_{i=1}^{M}S_i$, which can be formalized using the techniques provided in \citep[App. D]{Cai23}, as summarized below.  

\begin{lem}\label{lem:noisy_horizon}
    {\em (Adapted from \citep[App.~D]{Cai23} )} Let $f=\sum_{i=1}^{M}S_i g_i(\xv)$ be the prescribed function with $M$ disjoint regions, where the $M$ signs have an uniform prior over all $2^M$ sign patterns.  Consider any (possibly adaptive) deterministic algorithm, with queries corrupted by independent noise $\Nc(0,\sigma^2)$.  If 
    \begin{equation}\label{eq:noisy_T}
        T< c_3\cdot M \cdot \max \Big\{1,\frac{\sigma^2}{\eta^2}\Big\}
    \end{equation}
    for sufficiently small constant $c_3$, then with at least constant probability, there exists at least a constant fraction of $M$ regions indexed by  $i=1,\dotsc,M$ such that $S_i$ has strictly positive posterior variance given the $T$ samples.
\end{lem}
This lemma indicates that, conditioned on the observed samples, when $T$ satisfies \eqref{eq:noisy_T}, at least a constant fraction of these indices have strictly positive posterior variance.  In addition, the $S_i$ are known to be conditionally independent given the $T$ samples,  due to the use of disjoint bumps in the function class \citep[Lem.~5]{Cai23}.  Combining these observations and using  \eqref{eq:nc_value}, it follows that with constant probability the posterior variance of estimating $Z$ satisfies
\begin{equation}\label{eq:var_z}
    \var_T[Z] = \var_T\Big[1+\Delta Z \sum_{i=1}^{M}S_i\Big] = (\Delta Z)^2 \var_T\Big[\sum_{i=1}^{M}S_i\Big] = \Omega\big(M (\Delta Z)^2 \big),
\end{equation}
where $\var_T[\cdot]$ denotes variance given the $T$ samples, and we subsequently define $\EE_T[\cdot]$ analogously. 

Recall from Appendix \ref{app:class} that $ M= \lfloor \frac{1}{w}\rfloor^d$.  By substituting the lower bound on $\Delta Z$ in \eqref{eq:deltaz_lb} into \eqref{eq:var_z}, the lower bound for the posterior standard deviation is:
\begin{equation}
    \Omega\Big(\Delta Z \sqrt{M}\Big) = \Omega\Big(c_1 w^d (e^{c_2\lambda\eta}-1)\sqrt{M}\Big) = \Omega\Big(c_1(e^{c_2\lambda\eta}-1)M^{-\frac{1}{2}}\Big) = \Omega\Big(\lambda\eta M^{-\frac{1}{2}}\Big). \label{eq:posterior_std}
\end{equation}

We now argue that in the limit of large $M$ (which grows simultaneously with $T$), the mean absolute error of estimating \eqref{eq:nc_value} will be proportional to the above posterior standard deviation. To see this, we apply the central limit theorem for independent (but not necessarily identically distributed) Bernoulli random variables $S_i$, which yields that the following quantity converges to the standard Gaussian distribution as $M\to\infty$:
\begin{equation}
    \frac{Z-\EE_T[Z]}{\sqrt{\var_T[Z]}} = \frac{1+\Delta Z\sum_{i=1}^{M}S_i - \EE_T[1+\Delta Z\sum_{i=1}^{M}S_i]}{\Delta Z \sqrt{\var_T[\sum_{i=1}^{M}S_i]}} = \frac{\sum_{i=1}^{M}S_i - \EE_T[\sum_{i=1}^{M}S_i]}{\sqrt{\var_T[\sum_{i=1}^{M}S_i]}} \sim \Nc(0,1),
\end{equation}
where we have substituted \eqref{eq:var_z}.  Namely, $Z$ is asymptotically distributed as $\Nc(\EE_T[Z], \var_T[Z])$, and thus no matter what value the algorithm outputs as the estimate, there is always a constant probability that the correct value differs from it by $\Theta( \sqrt{\var[Z]} )$.\footnote{In particular, outputting an estimate of $\EE_T[Z]$ is asymptotically optimal, but $|Z - \EE_T[Z]|$ asymptotically follows a folded normal distribution, whose mean is $\sqrt{\frac{2\var_T[Z]}{\pi}}$.}

\subsubsection{Analysis of Noiseless Lower Bound}
The noiseless lower bounds can be obtained from the first term in Lemma \ref{lem:noisy_horizon} with $\sigma^2=0$, and accordingly considering $T = \Theta(M)$.  Recalling the relation between $\eta$ and $M$ in \eqref{eq:eta_M}, it can be seen that $\lambda \eta = \Theta(T^{-\frac{\nu}{d}-\frac{1}{2}+c})$.  Then we have:
\begin{itemize}
    \item When $\lambda = \Theta(T^c)$ with $c\le \frac{\nu}{d}+\frac{1}{2}$:  the upper bound in \eqref{eq:nc_ub} is $O(e^{\lambda \eta})=O(1)$.  Therefore, we can directly use the posterior standard deviation in \eqref{eq:posterior_std} as the relative error, and get the lower bound $\epsilon=\Omega(\lambda\eta M^{-\frac{1}{2}}) = \Omega(T^{-\frac{\nu}{d}-1+c})$.  
    \item When $\lambda = \Theta(\log T)$:  we have that $e^{\lambda \eta} = (e^{\log T})^{\eta} = T^{\Theta\big(T^{-\frac{\nu}{d}-\frac{1}{2}}\big)}$, being asymptotically $1$.  The error is then as least $\Omega(T^{-\frac{\nu}{d}-1} \log T) $.
\end{itemize}

\subsubsection{Analysis of Noisy Lower Bound}\label{app:noisy_lb}

If the maximum of \eqref{eq:noisy_T} is achieved by the second term, $T=\Theta(M\frac{\sigma^2}{\eta^2})$, we get $\eta M^{-\frac{1}{2}}=\Theta(\sigma T^{-\frac{1}{2}} )$, so that substitution into \eqref{eq:posterior_std} simply gives
\begin{equation*}
    \epsilon = \Omega\big(\lambda \sigma T^{-\frac{1}{2}}\big),
\end{equation*}
under the condition $\lambda\eta\to 0$ (so that \eqref{eq:nc_ub} scales as $\Theta(1)$).  We now justify the range of $\lambda$ and $\sigma$ to achieve this.  By reorganizing $\eta=\Theta(M^{-\frac{\nu}{d}-\frac{1}{2}})$ in \eqref{eq:eta_M} as $\eta^2=\Theta(M^{-\frac{2\nu}{d}-1})$, we obtain $\frac{M}{\eta^2} = \Theta(M^{-\frac{2\nu+2d}{d}})$.  Combining this with $T=\Theta\big( M \frac{\sigma^2}{\eta^2}\big)$, we deduce
\begin{equation}
    M=\Theta\Big(\sigma^{-\frac{2d}{2\nu+2d}} T^{\frac{d}{2\nu+2d}} \Big),
\end{equation}
and hence 
\begin{equation}
    \eta=\Theta\Big(M^{-\frac{2\nu+d}{2d}}\Big)=\Theta\Big(\sigma^{\frac{2\nu+d}{2\nu+2d}} T^{-\frac{2\nu+d}{4\nu+4d}} \Big).
\end{equation}
Letting $\lambda=\Theta(T^c)$ and $\sigma=\Theta(T^a)$, this yields
\begin{equation}\label{eq:lamb_eta}
    \lambda\eta = \Theta\Big(T^c M^{-\frac{\nu}{d}-\frac{1}{2}}\Big) = \Theta\Big(T^{\frac{2\nu+d}{2\nu+2d}a} T^{c-\frac{2\nu+d}{4\nu+4d}}\Big).
\end{equation}
For convenience in simplifying \eqref{eq:lamb_eta}, we assume $a\le \frac{1}{2}$,\footnote{This could potentially be generalized to certain cases with $a>\frac{1}{2}$, but we believe that this would be of less interest, since $\Theta(\sqrt{T})$ already amounts to an unusually high noise level.} so that as long as $c<0$, \eqref{eq:lamb_eta} reduces to $\lambda\eta=O(1)$, and consequently, \eqref{eq:nc_ub} scales as $\Theta(1)$.  Therefore, we have when $\lambda\to 0$ (e.g. $\lambda=\Theta(T^c)$ with $c < 0$), the lower bound is simply
\begin{equation*}
    \epsilon = \Omega\Big(\lambda T^{-\frac{\nu}{d}-1} + \sigma \lambda T^{-\frac{1}{2}}\Big).
\end{equation*}

On the other hand, when $\lambda=\Theta(T^c)$ with $c\ge 0$, in order to attain $\lambda\eta=\Theta(1)$, it is sufficient that $c\le (\frac{1}{2}-a) \frac{2\nu+d}{2\nu+2d}$.  Together with the noiseless term, the value of $c$ is now constrained to values in $\big[0, \min\{\frac{\nu}{d}+\frac{1}{2}, (\frac{1}{2}-a) \frac{2\nu+d}{2\nu+2d}\}\big]$, yielding the following lower bound:
\begin{equation}\label{eq:noisy_lower}
    \epsilon = \Omega\Big(T^{-\frac{\nu}{d}-1+c} + \sigma T^{-\frac{1}{2}+c}\Big), \quad c\in \Big[0, \min\Big\{\frac{\nu}{d}+\frac{1}{2}, \big(\frac{1}{2}-a\big) \frac{2\nu+d}{2\nu+2d}\Big\}\Big].
\end{equation}

By comparing $\frac{\nu}{d}+\frac{1}{2}$ with $(\frac{1}{2}-a) \frac{2\nu+d}{2\nu+2d}$, we find that
\begin{align*}
    \frac{\nu}{d}+\frac{1}{2} &\le \Big(\frac{1}{2}-a\Big) \frac{2\nu+d}{2\nu+2d} \\
     \Longleftrightarrow \frac{2\nu+d}{2\nu+2d}a &\le \frac{2\nu+d}{4\nu+4d}-\frac{1}{2}-\frac{\nu}{d} \\
     \Longleftrightarrow \frac{2\nu+d}{2\nu+2d}a &\le \frac{2\nu d+d^2-2\nu d-2d^2-4\nu^2-4\nu d}{(4\nu+4d)d} \\
     \Longleftrightarrow \frac{2\nu+d}{2\nu+2d}a &\le \frac{-(2\nu+d)^2}{(4\nu+4d)d} \\
     \Longleftrightarrow a &\le -\frac{2\nu +d}{2d} =-\frac{\nu}{d}-\frac{1}{2},
\end{align*}
meaning that when $\sigma=\Theta(T^{a})$ with $a\le -\frac{\nu}{d}-\frac{1}{2}$, the error in \eqref{eq:noisy_lower} is dominated by its first term, and the lower bound becomes the noiseless lower bound (and we are free to choose any $c\le \frac{\nu}{d}+\frac{1}{2}$).  In fact, the same condition can be obtained by directly comparing the two powers $-\frac{\nu}{d}-1$ and $a-\frac{1}{2}$.  

On the other hand, when $a> -\frac{\nu}{d}-\frac{1}{2}$, the dominating term is $\Omega(\sigma T^{-\frac{1}{2} +c})$ with $c\le (\frac{1}{2}-a) \frac{2\nu+d}{2\nu+2d}$.  In particular, when 
\begin{equation}\label{eq:c-a}
    c= \Big(\frac{1}{2}-a\Big) \frac{2\nu+d}{2\nu+2d},
\end{equation}
which is the maximum allowed value, the lower bound can be simplified as
\begin{equation}
    \Omega\Big(T^{a-\frac{1}{2} + (\frac{1}{2}-a)\frac{2\nu+d}{2\nu+2d}}\Big) = \Omega\Big(T^{ (a-\frac{1}{2}) \frac{d}{2\nu+2d}}\Big). \label{eq:lb_with_a}
\end{equation}

Therefore, we have the following lower bound instances by choosing different noise levels of interest (still with $a>-\frac{\nu}{d}-\frac{1}{2}$) in \eqref{eq:lb_with_a}:
\begin{itemize}
    \item We first consider negative $a$ cases. For example, with $a=-\frac{1}{4}$ and $a=-\frac{1}{2}$, the results are already shown in Table \ref{tab:results} (though with a general choice of $\lambda$).
    \item Another special case is by choosing $a = -\frac{\nu}{d}+\frac{1}{2}$ (assuming $\nu>\frac{d}{2}$, so we have $a<0$), for which the lower bound becomes \fbox{$\Omega(T^{-\frac{\nu}{2\nu+2d}})$} by substituting the value of $a$ into \eqref{eq:lb_with_a}.  From \eqref{eq:c-a}, this implies that the order of $\lambda$ is now at $\lambda=\Theta(T^{\frac{\nu}{d}\frac{2\nu+d}{2\nu+2d}})$.  For comparison, the existing BO noisy lower bound is $\epsilon=\Theta(\sigma^{\frac{2\nu}{2\nu+d}} T^{-\frac{\nu}{2\nu+d}})$ \citep{Sca17}. Thus by substituting $a=-\frac{\nu}{d}+\frac{1}{2}$, the BO noisy lower bound becomes $\Omega(T^{\frac{-2\nu+d}{2d}\frac{2\nu}{2\nu+d}-\frac{\nu}{2\nu+d}}) = \Omega(T^{-\frac{2\nu^2}{(2\nu+d)d}})$, indicating that NC is strictly harder than BO for these choices $a$ and $c$, since $\frac{\nu}{2\nu+2d}<\frac{2\nu^2}{(2\nu+2d)d}<\frac{2\nu^2}{(2\nu+d)d}$ when $\nu>\frac{d}{2}$.  In other words, while NC typically requires fewer samples than BO for small $\lambda$, in can require strictly more for large $\lambda$. 
    \item When $a=0$, we arrive at the {\em most studied case}, $\sigma=\Theta(1)$.  As already discussed in Section \ref{sec:lb}, the lower bound becomes \fbox{$\Omega(T^{ -\frac{d}{4\nu+4d}})$} (Table \ref{tab:results} also shows a general result) in the case that $\lambda=\Theta(T^{\frac{2\nu+d}{4\nu+4d}})$.  Similar to the above, the NC error can be higher than BO if $\nu\ge\frac{d}{2}$ (when $\sigma=\Theta(1)$, the BO lower bound is $\epsilon=\Omega(T^{-\frac{\nu}{2\nu+d}})$). 
    \item As an example of a positive $a$ value (i.e., large noise), setting $a=\frac{1}{4}$ gives \fbox{$\epsilon=\Omega\big(T^{-\frac{d}{8\nu+8d}}\big)$}.
    \item When $a=\frac{1}{2}$, we are restricted to choose $c=0$, and obtain \fbox{$\epsilon=\Omega(1)$}.  This corresponds to having an overly large noise level that prohibits accurate estimation.
    \item When $\lambda=\Theta(\log T)$ and $a\le \frac{1}{2}$, $e^{c_2\lambda\eta}$ scales as $\Theta(1)$ when $T\to \infty$, and thus the error becomes \fbox{$\Omega\big(T^{-\frac{\nu}{d}-1}\log T + \sigma T^{-\frac{1}{2}}\log T\big)$}.
\end{itemize}

\section{Proofs of Upper Bounds}

\subsection{Deriving the Variance of Relative Error}
With the calculation of the estimated residual $\Rhat$ in Algorithm \ref{alg:two-phase-nc} (line 12), the true residual value $R$ (given $\mu_{T/2}$) is determined by
\begin{equation}\label{eq:true_R}
    R:= \frac{Z}{\Zhat_1} = \frac{\int_D e^{-\lambda f(\xv)}d\xv}{\int_D e^{-\lambda\mu_{T/2}(\xv)}d\xv} = \int_D e^{\lambda\mu_{T/2}(\xv) - \lambda f(\xv)} \frac{e^{-\lambda\mu_{T/2}(\xv)}}{\int_D e^{-\lambda\mu_{T/2}(\xv)}d\xv} d\xv = \EE_{\rho(\mu)} [e^{\lambda\mu_{T/2}(\xv) - \lambda f(\xv)}],
\end{equation}
where we denote $\rho(\mu) = \frac{e^{-\lambda\mu_{T/2}(\xv)}}{\int_D e^{-\lambda\mu_{T/2}(\xv)}d\xv} $.  It is easy to verify that the expectation of $\Rhat$ w.r.t.~the second batch samples (denoted by $\EE_2[\cdot]$) is unbiased:
\begin{align}
    \EE_2[\Rhat] &=  \EE_{\rho(\mu)} \bigg[\EE_{z_t} \Big[\frac{2}{Te^{\lambda^2\sigma^2/2}}\sum_{t=T/2+1}^{T} e^{\lambda\mu_{T/2}(\xv_t)-\lambda f(\xv_t)-\lambda z_t} \Big] \bigg] \nonumber\\
    &= \frac{2}{T}\sum_{t=1}^{T/2} \EE_{\rho(\mu)}\bigg[e^{\lambda\mu_{T/2}(\xv_t)-\lambda f(\xv_t)} \bigg] \nonumber\\
    &= R, \label{eq:unbias}
\end{align}
where the second step uses the fact that the moment generating function of a normal random variable $z_t\sim\Nc(0,\sigma^2)$ is $\EE_{z_t}[e^{\lambda z_t}]= e^{\lambda^2\sigma^2/2}$.  Moreover, taking the samples from the first batch into consideration, we can upper bound the relative error by relative variance, as shown below.
\begin{lem}\label{lem:eps_var}
(Upper Bound via Variance)  Under the setup and notation of Algorithm \ref{alg:two-phase-nc}, for any $\delta_2>0$, with probability at least $1-\delta_2$ with respect to the second batch of samples, the following holds:
\begin{equation}
    \epsilon \le \sqrt{\frac{1}{\delta_2}\var_2\Big[\frac{\Rhat}{R}\Big]},
\end{equation}
where $\var_2[\cdot]$ denotes variance based on the samples in the second batch alone.

\end{lem}
\begin{proof}
    The proof is a straightforward application of Chebyshev's inequality. We first observe that:
    \begin{equation*}
        \frac{\Zhat}{Z} = \frac{\Zhat_1\Rhat}{Z} = \frac{\Rhat}{R},
    \end{equation*}
    and from \eqref{eq:unbias}, we have:
    \begin{equation*}
    \EE_2\Big[\frac{\Rhat}{R}\Big] = 1.
    \end{equation*}
    As $\frac{\Rhat}{R}$ is a random variable for the second batch samples, we can use Chebyshev's inequality to show that, for any $\delta_2 > 0$:
    \begin{equation*}
        \PP\bigg[\Big|\frac{\Rhat}{R} - \EE_2\Big[\frac{\Rhat}{R}\Big]\Big| \ge \frac{1}{\sqrt{\delta_2}}\sqrt{\var_2\Big[\frac{\Rhat}{R}\Big]}\bigg] \le \delta_2.
    \end{equation*}
\end{proof}

By fixing the first batch samples and hence $R$, we can derive the variance of $\frac{\Rhat}{R}$ with respect to second batch samples:
\begin{align}
    \var_2\Big[\frac{\Rhat}{R}\Big] &= \var_2\bigg[\frac{\frac{2}{Te^{\lambda^2\sigma^2/2}}\sum_{t=T/2+1}^{T}e^{\lambda\mu_{T/2}(\xv_t)-\lambda f(\xv_t)-\lambda z_t}}{R}\bigg] \nonumber\\
    &= \frac{4}{T^2 e^{\lambda^2\sigma^2} R^2} \var_2\bigg[\sum_{t=T/2+1}^{T}e^{\lambda\mu_{T/2}(\xv_t)-\lambda f(\xv_t)-\lambda z_t}\bigg] \nonumber\\\
    &= \frac{4}{T^2 e^{\lambda^2\sigma^2}  R^2} \sum_{t=T/2+1}^{T} \var_2\big[e^{\lambda\mu_{T/2}(\xv_t)-\lambda f(\xv_t)-\lambda z_t}\big] \nonumber\\
    &= \frac{4}{T^2 e^{\lambda^2\sigma^2}  R^2} \sum_{t=T/2+1}^{T} \Big( \EE_2\big[e^{2\lambda\mu_{T/2}(\xv_t)-2\lambda f(\xv_t)-2\lambda z_t}\big] - \EE_2\big[e^{\lambda\mu_{T/2}(\xv_t)-\lambda f(\xv_t)-\lambda z_t}\big]^2 \Big) \nonumber\\
    &= \frac{2}{T}\frac{\EE_{z_t}\big[e^{2\lambda z_t}\big] \EE_{\rho(\mu)}\big[e^{2\lambda\mu_{T/2}(\xv)-2\lambda f(\xv)}\big] - \EE_{z_t}\big[e^{\lambda z_t}\big]^2\EE_{\rho(\mu)}\big[e^{\lambda\mu_{T/2}(\xv)-\lambda f(\xv)}\big]^2}{e^{\lambda^2\sigma^2} R^2}  \nonumber\\
    &= \frac{2}{T} \bigg(\frac{e^{\lambda^2\sigma^2} \EE_{\rho(\mu)}\big[e^{2\lambda\mu_{T/2}(\xv)-2\lambda f(\xv)}\big]}{\EE_{\rho(\mu)}\big[e^{\lambda\mu_{T/2}(\xv)-\lambda f(\xv)}\big]^2} - 1\bigg), \label{eq:error_upper_mid}
\end{align}
where the third and fifth step use the independence and identical properties of $\xv_t$ and $z_t$ across $t$, and the last step follows by noting that $\EE_{z_t}[e^{\lambda z_t}]= e^{\lambda^2\sigma^2/2}$ (since $z_t \sim \Nc(0,\sigma^2)$) and using \eqref{eq:true_R}.

\subsection{Variance Bounds Based on $L^{\infty}$ Norm}
The following lemma will be useful for further upper bounding \eqref{eq:error_upper_mid} using $L^{\infty}$-norm based confidence bounds on $f$.

\begin{lem} \label{lem:chern}
    {\em (Hoeffiding's Lemma; \citep[Lem. 2.2]{Bou04}~)}
    Let $Y$ be a random variable taking values in a bounded interval $[a,b]$.  Then for all $\lambda>0$, 
    \begin{equation}  \label{eq:hoeff}
        \EE[e^{\lambda (Y-\EE[Y])}] \le e^{\frac{\lambda^2(a-b)^2}{8}}.
    \end{equation}
\end{lem}
We use the above lemma with the replacement of $Y$ by $\mu_{T/2}(\xv) - f(\xv)$ (with the randomness being $\xv\sim \rho(\mu)$), and the fact that $\mu_{T/2}(\xv) - f(\xv)\in [-\|\mu_{T/2}(\xv)-f(\xv)\|_{L^{\infty}},\|\mu_{T/2}(\xv)-f(\xv)\|_{L^{\infty}} ]$, and get
\begin{align}
    \EE_{\rho(\mu)}[e^{\lambda (\mu_{T/2}(\xv)-f(\xv))}] &\le e^{\EE_{\rho(\mu)}[\lambda\mu_{T/2}(\xv)-\lambda f(\xv)] + \frac{\lambda^2\|\mu_{T/2}(\xv)-f(\xv)\|_{L^{\infty}}^2}{2}}.  \label{eq:hoeff_1}
\end{align}

We now bound the denominator and numerator of \eqref{eq:error_upper_mid} separately.  By Jensen's inequality, the denominator is easily lower bounded as follows:
\begin{equation}
    \EE_{\rho(\mu)}\big[e^{\lambda\mu_{T/2}(\xv)-\lambda f(\xv)}\big]^2  \ge e^{2\EE_{\rho(\mu)}[\lambda\mu_{T/2}(\xv)-\lambda f(\xv)]}. \label{eq:error_denom}
\end{equation}
For the numerator of \eqref{eq:error_upper_mid}, we replace $\lambda$ by $2\lambda$ in \eqref{eq:hoeff_1}, and obtain
\begin{align}
    e^{\lambda^2\sigma^2} \EE_{\rho(\mu)}\big[e^{2\lambda\mu_{T/2}(\xv)-2\lambda f(\xv)}\big] \le  e^{\lambda^2\sigma^2 + 2 \EE_{\rho(\mu)}[\lambda\mu_{T/2}(\xv)-\lambda f(\xv)] + 2\lambda^2 \|\mu_{T/2}(\xv)-f(\xv)\|_{L^{\infty}}^2}. \label{eq:error_numer}
\end{align}
Combining \eqref{eq:error_upper_mid}, \eqref{eq:error_denom} and \eqref{eq:error_numer}, we obtain
\begin{align} 
    \var_2\Big[\frac{\Rhat}{R}\Big] &\le \frac{2}{T} \Big( e^{\lambda^2\sigma^2 + 2\lambda^2\|\mu_{T/2}(\xv) - f(\xv)\|_{L^{\infty}}^2} - 1\Big). \label{eq:var_ub}
\end{align}
In subsequent sections, we will specify the range of $\lambda$ in order to approximate \eqref{eq:var_ub} via 
\begin{equation}\label{eq:ex}
    e^x = 1+ O(x), \quad \text{when } x = O(1).    
\end{equation}

\subsection{Noiseless Upper Bounds (Proof of Theorem \ref{thm:noiseless_upper})}
In the noiseless case, the following lemma shows that the posterior variance provides an upper bound on the $L^{\infty}$ error for a fixed target function.
\begin{lem}\label{lem:noiseless_conf_bound}
    {\em (Noiseless Confidence Intervals \citep[Cor.~3.11]{kan18})}
    For any $f\in\mat(B)$, after $t$ noiseless observations, $L_t(\xv)\le f(\xv)\le U_t(\xv)$ holds with probability one for any $\xv\in D$, where
    \begin{align*}
        U_t(\xv) &= \mu_{t-1}(\xv) + B \sigma_{t-1}(\xv), \\
        L_t(\xv) &= \mu_{t-1}(\xv) - B \sigma_{t-1}(\xv),
    \end{align*}
    and where $\mu_{t-1}(\cdot)$ and $\sigma_{t-1}(\cdot)$ are given in \eqref{eq:posterior_mean}--\eqref{eq:posterior_variance} with $\xi = 0$.
\end{lem}
From the above lemma, for Mat\'ern-$\nu$ kernel, it has been derived that maximum variance sampling gives (e.g. see \citep[Sec. 3.4]{Wen05})
\begin{equation}\label{eq:Linfty_nosig}
    \|f(\xv) - \mu_{T/2}(\xv)\|_{L^{\infty}} = O(T^{-\frac{\nu}{d}}).
\end{equation}
Therefore, when the order of $\lambda$ is at most $\Theta(T^{\frac{\nu}{d}})$, which guarantees that $\lambda\|f(\xv) - \mu_{T/2}(\xv)\|_{L^{\infty}}$ is bounded by $O(1)$ according to \eqref{eq:Linfty_nosig}, we can simplify \eqref{eq:var_ub} (with $\sigma=0$) using \eqref{eq:ex} as follows:
\begin{equation*}
     \var_2\Big[\frac{\Rhat}{R}\Big]= O\big(\lambda^2 T^{-1}  \|f(\xv) - \mu_{T/2}(\xv)\|^2_{L^{\infty}}\big)  = O\Big(\lambda^2 T^{-\frac{2\nu}{d}-1}\Big).
\end{equation*}
Hence, by Lemma \ref{lem:eps_var}, we obtain with probability at least $1-\delta_2$ that
\begin{equation*}
     \epsilon = O\Big(\lambda T^{-\frac{\nu}{d}-\frac{1}{2}}\Big).
\end{equation*}
Specifically, the following upper bounds follow:
\begin{itemize}
    \item When $\lambda = \Theta(T^{c})$ with $c\le\frac{\nu}{d}$, $\epsilon = O\big(T^{-\frac{\nu}{d}-\frac{1}{2}+c} \big)$.
    \item When $\lambda = \Theta(\log T)$, $\epsilon = O\big(T^{-\frac{\nu}{d}-\frac{1}{2}}\log T \big)$.
\end{itemize}
We have thus proved Theorem \ref{thm:noiseless_upper}.

\subsection{Noisy Upper Bounds (Proof of Theorem \ref{thm:noisy_upper})}
In the case of noisy observations, we begin with a useful noisy confidence interval presented in the BO literature \citep{Vak21}.  

\begin{lem}\label{lem:fmu_infty}
    {\em (Confidence Bound for Non-Adaptive Sampling \citep{Vak21})}
    For $f\in\mat(B)$ and any $\delta_1>0$, after $T$ noisy non-adaptive observations (i.e., all inputs are chosen before observing any outputs), the following holds with probability at least $1-\delta_1$:
    \begin{equation}
        \|f-\mu_{T/2}\|_{L^{\infty}} = O\bigg(\sqrt{\frac{\gamma_T}{T}} \Big(B+\frac{\sigma}{\xi}\sqrt{d\log T + 2\log\frac{1}{\delta_1})}\bigg).
    \end{equation}
\end{lem}
For the Mat\'ern-$\nu$ kernel, the maximum information gain $\gamma_T$ is bounded by \citep{Vak21a}
\begin{equation*}
    \gamma_T=O\big(T^{\frac{d}{2\nu +d}}(\log{T})^{\frac{2\nu}{2\nu +d}}\big).
\end{equation*}
Thus, from the above lemma, by choosing $\delta_1=\frac{1}{T^{\alpha}}$ for an arbitrary constant $\alpha > 0$, the $L^{\infty}$-error simplifies to
\begin{equation}\label{eq:fmu_infty}
    \|f-\mu_{T/2}\|_{L^{\infty}} = O\Big(T^{-\frac{\nu}{2\nu+d}}(\log T)^{\frac{\nu}{2\nu+d}} + \sigma T^{-\frac{\nu}{2\nu+d}}(\log T)^{\frac{4\nu+d}{4\nu+2d}}\Big),
\end{equation}
where $B$, $d$, $\xi$, and $\alpha$ are hidden as constants.

\subsubsection{Noisy Relative Error When $\lambda\sigma\to 0$}\label{sec:noisy0}

In the noisy setting, there exits an $e^{\lambda^2\sigma^2}$ term in \eqref{eq:var_ub}, so in order to make the ``$*$'' term of $\exp[*]$ to be non-growing, it is required that $a+c\le 0$ when $\lambda=\Theta(T^c)$ and $\sigma=\Theta(T^a)$, so that $\lambda\sigma=\Theta(1)$.  Thus, when $\lambda\sigma + \lambda\|f-\mu_{T/2}\|_{L^{\infty}}=O(1)$, we can substitute \eqref{eq:fmu_infty} into \eqref{eq:var_ub}, and simplify according to \eqref{eq:ex}.  Finally, with Lemma \ref{lem:eps_var}, the following holds with probability at least $(1-\delta_2)\big(1-\frac{1}{T^c}\big)$:
\begin{equation*}
    \epsilon = O\big(\lambda\sigma T^{-\frac{1}{2}} + \lambda T^{-\frac{1}{2}}\|f-\mu_{T/2}\|_{L^{\infty}} \big) = O\big(\lambda T^{-\frac{\nu}{2\nu+d}-\frac{1}{2}}(\log T)^{\frac{\nu}{2\nu+d}} + \lambda \sigma T^{-\frac{1}{2}}\big),
\end{equation*}
where an additional noisy term $O(\sigma T^{-\frac{\nu}{2\nu+d}-\frac{1}{2}}(\log T)^{\frac{4\nu+d}{4\nu+2d}})$ is absorbed into $O(\cdot)$ since it is dominated by $O(\sigma T^{-\frac{1}{2}})$. Hence, the following upper bounds are established:
\begin{itemize}
    \item When $\lambda=\Theta(T^c)$ with $a+c \le 0$ and $c<\frac{\nu}{2\nu+d}$, the condition $\lambda\sigma + \lambda\|f-\mu_{T/2}\|_{L^{\infty}}=O(1)$ is satisfied by \eqref{eq:fmu_infty}.  Note we exclude $c$ from being exactly $\frac{\nu}{2\nu+d}$, since in that case the remaining $\log T$ term will grow unbounded.  Consequently, we have $\epsilon = O\big(T^{-\frac{\nu}{2\nu+d}-\frac{1}{2}+c}(\log T)^{\frac{\nu}{2\nu+d}} + \sigma T^{-\frac{1}{2}+c}\big)$.
    \item When $\lambda=\Theta(\log T)$ and $a\le 0$, $\epsilon=O\big(T^{-\frac{4\nu+d}{4\nu+2d}}(\log T)^{\frac{3\nu+d}{2\nu+d}} + \sigma T^{-\frac{1}{2}}\log T\big)$.
\end{itemize}
This gives the first part of Theorem \ref{thm:noisy_upper}.

\subsubsection{Noisy Relative Error When $\lambda \sigma \to \infty$}

As hinted in previous sections, the $e^{\lambda^2\sigma^2}$ term caused by Algorithm \ref{alg:two-phase-nc} may be prohibitively large unless $\lambda\sigma\to 0$.  For other choices of $\lambda$ or $\sigma$ such that $\lambda\sigma\to\infty$, we consider the simpler strategy of only utiliizing the intermediate estimate $\Zhat_1$ (now using all $T$ samples instead of only $T/2$).  In this case, the relative error is simply
\begin{align}
    \epsilon &= \frac{\int_D e^{-\lambda \mu_{T}(\xv)}d\xv}{\int_D e^{-\lambda f(\xv)}d\xv} - 1 \nonumber\\
    &\le \frac{\int_D e^{-\lambda f(\xv) + \lambda\|f(\xv) -\mu_T(\xv)\|_{L^{\infty}} } d\xv}{\int_D e^{-\lambda f(\xv)}d\xv} - 1 \nonumber\\
    &= e^{\lambda \|f(\xv) -\mu_T(\xv)\|_{L^{\infty}}} - 1, \label{eq:ub_sig_2}
\end{align}
where the second last step is due to $-\|f(\xv) -\mu_T(\xv)\|_{L^{\infty}}\le \mu_T(\xv) -f(\xv) \le \|f(\xv) -\mu_T(\xv)\|_{L^{\infty}} $.  Continuing  with \eqref{eq:ub_sig_2}, the upper bounds directly follow from \eqref{eq:fmu_infty} if $\lambda\|f-\mu_T\|_{L^{\infty}}=O(1)$.  Therefore, we can use \eqref{eq:ex} and \eqref{eq:fmu_infty} to establish that:
\begin{itemize}
    \item When $\lambda=\Theta(T^{c})$ with $a+c <\frac{\nu}{2\nu+d}$ and $c <\frac{\nu}{2\nu+d}$, the condition $\lambda\|f-\mu_T\|_{L^{\infty}}=O(1)$ is satisfied by \eqref{eq:fmu_infty}, following a similar approach as in the previous subsection.  Then, $\epsilon=O\big(T^{-\frac{\nu}{2\nu+d}+c}(\log T)^{\frac{\nu}{2\nu+d}} + \sigma T^{-\frac{\nu}{2\nu+d}+c}(\log T)^{\frac{4\nu+d}{4\nu+2d}}\big)$.
    \item When $\lambda=\Theta(\log T)$ and $a< \frac{\nu}{2\nu+d}$, $\epsilon=O\big(T^{-\frac{\nu}{2\nu+d}}(\log T)^{\frac{3\nu+d}{2\nu+d}} + \sigma T^{-\frac{\nu}{2\nu+d}}(\log T)^{\frac{8\nu+3d}{4\nu+2d}}\big)$.
\end{itemize}
This gives the second part of Theorem \ref{thm:noisy_upper}.

\section{Further Comparison to Existing Work} \label{sec:comparison}

Here we outline in more detail the relationships and distinctions between our work and that of \citep{Hol23}:
\begin{itemize}
    \item The most significant distinction lies in our consideration of the impact of noise when estimating the normalizing constant. In contrast to their analysis, which is specific to the noiseless setting, our approach is capable of handling both fixed and varying $\sigma^2$.
    \item In terms of information-based complexities, their lower bound in Proposition 10 is primarily suited to large $\lambda$ (or the \emph{optimization regime} in their terminology),\footnote{More precisely, there is no explicit $\lambda$ in \citep{Hol23}, but they attain an equivalent effect by scaling the function.} in which a tight result can be obtained.  For $\lambda\to 0$, or the \emph{integration regime}, their lower bound grows as $\Omega(\log T)$.  We partially address the resulting gaps, and develop proof techniques that can handle both noisy and noiseless lower bounds in a unified manner, encompassing both optimization and integration regimes.
    \item While their results focus on $m$-differentiable functions with $L^\infty$ norm, we focus on Mat\'ern RKHS equipped with $L^2$ norm, which is closely related to GP bandits. In more detail, \citep{Hol23} focuses on the following function class :
    \begin{equation}\label{eq:m_diff}
        \Big\{f\in C^m(D), \|f\|_{C^m}=\sup_{\balpha:|\balpha|\le m} \|\partial^{\balpha} f\|_{L^{\infty}}\le B \Big\},
    \end{equation}
    where $m$ is an {\em integer}, the bold $\balpha$ is a $d$-dimensional multi-index with $|\balpha| = \sum^{d}_{i=1}\alpha_i$, and $\partial^{\balpha} f(\xv)$  is the weak derivative of order $\balpha$.  Meanwhile, the Mat\'ern RKHS is equivalent to the following Sobolev function class $W_2^s$ with a {\em fractional} order $s$, accompanied by an $L^2$ norm:
    \begin{equation}\label{eq:sob_s}
        \Big\{f\in L^2(\RR^d), \|f\|_{W_2^s} = \Big(\int_{\RR^d} (1+\|\xv\|_2^2)^s |\fhat(\bxi)|^2 d\xv \Big)^\frac{1}{2} \le B \Big\},
    \end{equation}
    where $\fhat(\bxi)$ stands for the Fourier transform of $f(\xv)$.  The above function class can be further defined over $D$ with a restriction argument.  To draw a parallel comparison to the integer-ordered case in \eqref{eq:m_diff}, it is possible to set $s$ in \eqref{eq:sob_s} as an integer, leading to a reduction of the following integer-order Sobolev class:
    \begin{equation}\label{eq:sob_s_int}
        \Big\{f\in L^2(D), \|f\|_{W_2^s}=\Big(\sum_{|\balpha|\le s} \|\partial^{\balpha} f\|_{L^2}^2\Big)^{\frac{1}{2}}\le B \Big\}.
    \end{equation}
    Consequently, the two types of results may not be directly transferable.
    \item In estimating the normalizing constant, our algorithm and the one in \citep{Hol23} share a similar two-step idea of trading worse Monte-Carlo computational complexity for a better overall convergence rate, with the distinction being that we utilize a GP surrogate rather than a piecewise constant function approximation.  This two-step idea has been widely used in numerical integration since \citep{Bak59}, and has been further developed in the literature on control functionals \citep{Rip09,Oat17} and BQ \citep{Cai23}.
    \item Regarding the two points mentioned above, it is important to highlight that while \citep{Hol23} has provided a proof of optimality for their algorithm when $\lambda\to\infty$, the same outcome does not hold for a similar algorithm in our case (where an extra $\sqrt{T}$ gap arises). The discrepancy arises from differences in the function classes used in the analysis -- the intermediate step of approximating the function well in $L^{\infty}$ norm is more suited to their function class, and less suited to ours, perhaps because the Mat\'ern RKHS is more closely related to $L^2$ function norms rather than $L^{\infty}$.
    \item \citep{Hol23} raises the open problem of whether there exists a fixed polynomial order $O(T^k)$ runtime (i.e, $k$ is independent of function smoothness $\nu$ and dimension $d$) to achieve near optimal convergence in the noiseless setting.  In our work, we demonstrate that in the noisy setting,\footnote{As previously mentioned, our current analysis based on worst-case noiseless function approximation incurs an additional error of $O(\sqrt{T})$, which circumvents us from achieving optimal convergence rates in the noiseless setting.  In the noisy setting, however, we can still establish tightness in certain cases that the noise-dependent term dominates the final error bound.}  there at least exists asymptotically optimal algorithms with runtime $O(T^3 + T^2 \tau_{\rm ker} + T \tau_{\rm samp})$, where:
    \begin{itemize}
        \item[$\cdot$] $\tau_{\rm ker}$ is the kernel evaluation time;
        \item[$\cdot$] $\tau_{\rm samp}$ is the time taken to perform a single sampling step in Line \ref{line:samp} of Algorithm \ref{alg:two-phase-nc}.
    \end{itemize}
    See, for example, the optimal noisy convergence in Table \ref{tab:results} when $\sigma=\Theta(T^{-\frac{1}{4}})$.  This is primarily attributed to the computational efficiency of a GP surrogate surpasses that of the piecewise constant function, effectively mitigating the curse of dimensionality.
    \item We will see experimentally in Section \ref{sec:exp} that their piecewise-constant approximation technique, while effective in large-budget settings, can be considerably less suited to low-budget settings (smaller $T$) compared to our Algorithm \ref{alg:two-phase-nc}.  That is, our approach harnesses the well-known strong sample efficiency of GP methods.
\end{itemize}

\section{Additional Experimental Results}\label{app:add_exp}
Along with the benchmark functions mentioned in Section \ref{sec:exp_analytic}, we consider an additional 2D function defined by $f(\xv) = -\sin(3\|\xv\|)^2-\xv^T S\xv$ with $S=\big[\begin{smallmatrix}
  1 & 0.5\\
  0.5 & 1
\end{smallmatrix}\big]$.  Estimating the normalizing constant of this function is equivalent to estimating the integration of the Hennig function, i.e., $\int e^{-\sin(3\|\xv\|)^2-\xv^T S\xv} d\xv$.  The results are shown in Figures \ref{fig:add_mlp} and \ref{fig:add_analytic}, where we observe similar findings to those discussed in Section \ref{sec:exp}.

\begin{figure}
    \centering
    \begin{subfigure}{\columnwidth}
        \centering
        \includegraphics[width=\sizecustom]{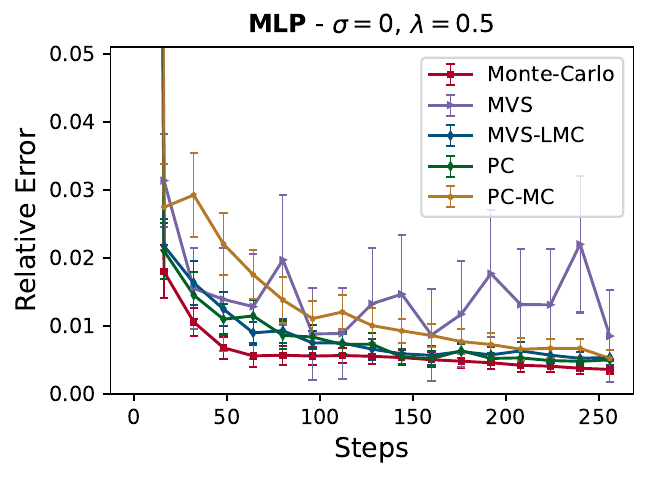}
        \includegraphics[width=\sizecustom]{figs/mlp/mlp-0-5.pdf}
        \includegraphics[width=\sizecustom]{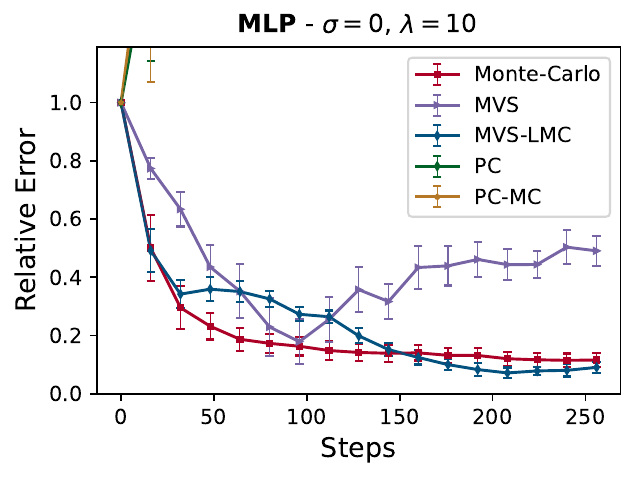}
    \end{subfigure}
    \vskip 0.2in
    \begin{subfigure}{\columnwidth}
        \centering
        \includegraphics[width=\sizecustom]{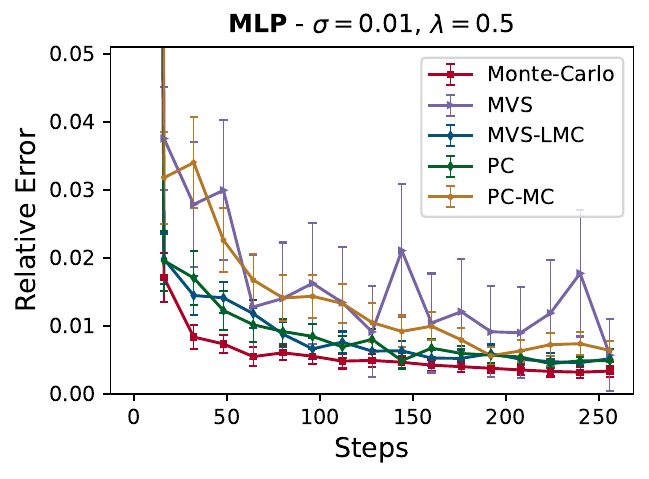}
        \includegraphics[width=\sizecustom]{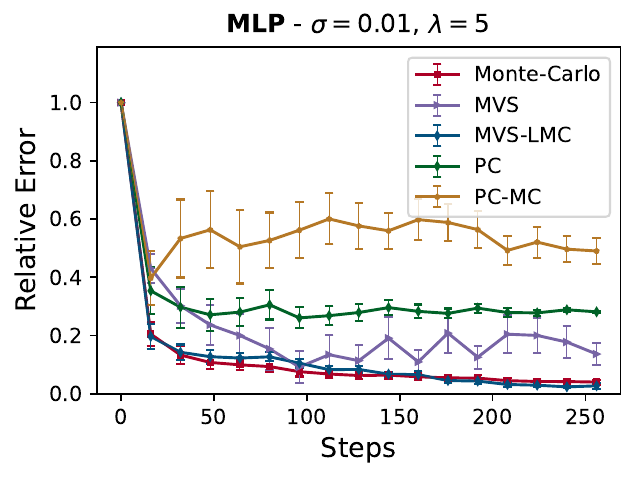}
        \includegraphics[width=\sizecustom]{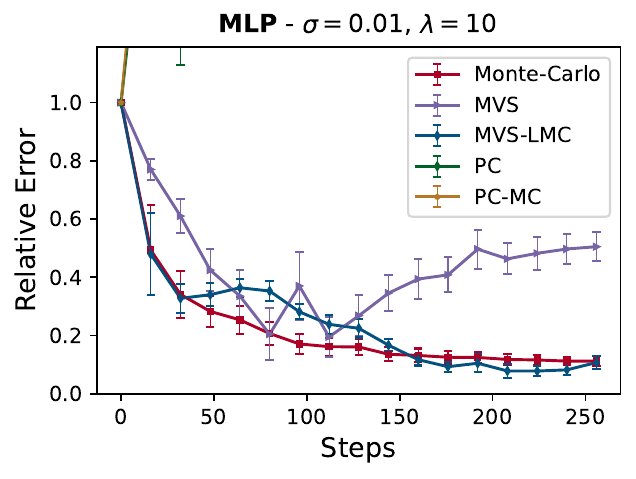}
    \end{subfigure}
    \vskip 0.2in
    \begin{subfigure}{\columnwidth}
        \centering
        \includegraphics[width=\sizecustom]{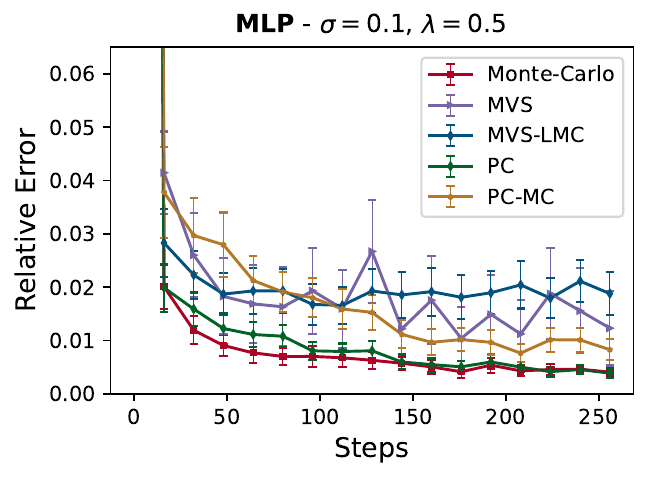}
        \includegraphics[width=\sizecustom]{figs/mlp/mlp-0.1-5.pdf}
        \includegraphics[width=\sizecustom]{figs/mlp/mlp-0.1-10.pdf}
    \end{subfigure}
    \caption{Additional results for MLP. \label{fig:add_mlp}}
\end{figure}

\begin{figure}
    \centering
    \begin{subfigure}{\columnwidth}
        \centering
        \includegraphics[width=\sizecustom]{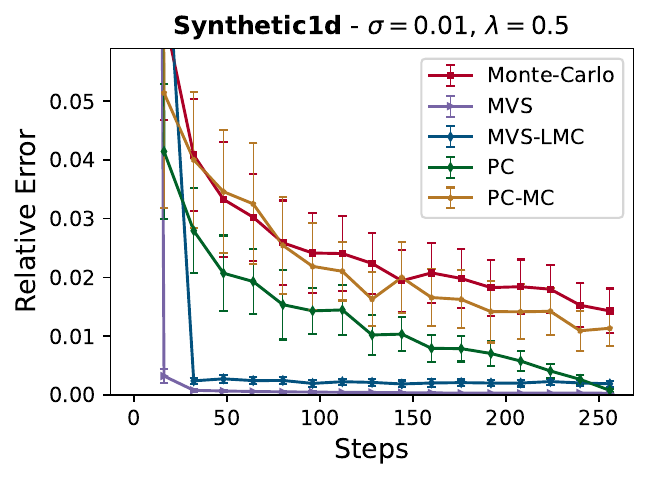}
        \includegraphics[width=\sizecustom]{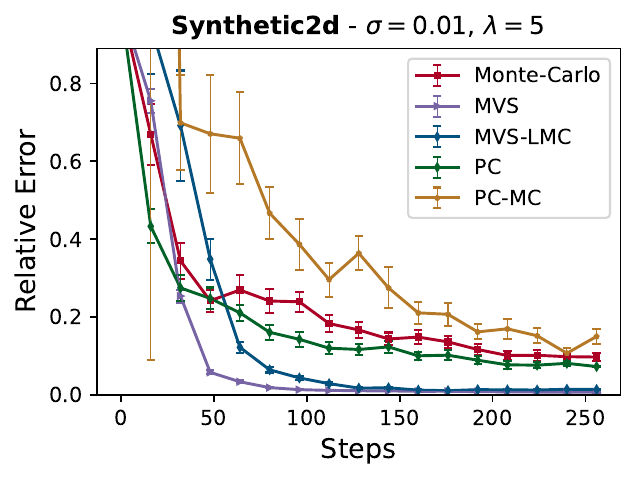}
        \includegraphics[width=\sizecustom]{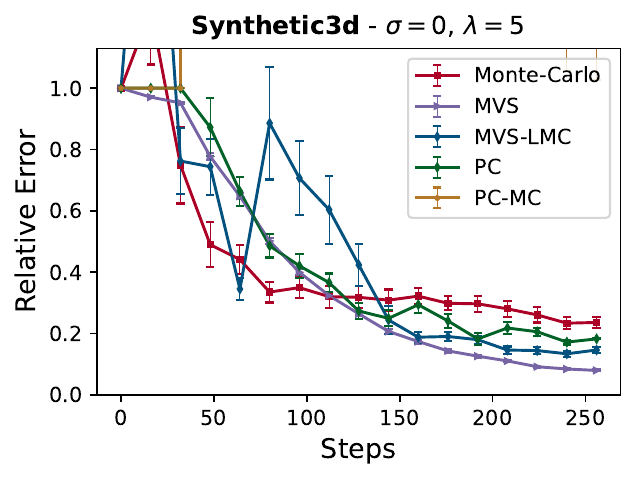}
    \end{subfigure}
    \vskip 0.2in
    \begin{subfigure}{\columnwidth}
        \centering
        \includegraphics[width=\sizecustom]{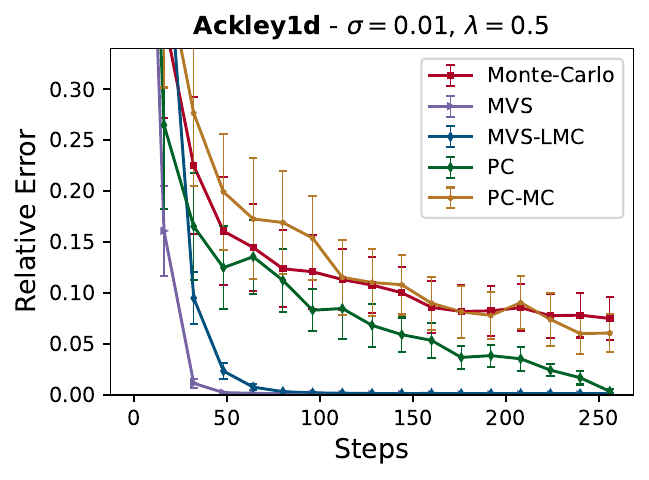}
        \includegraphics[width=\sizecustom]{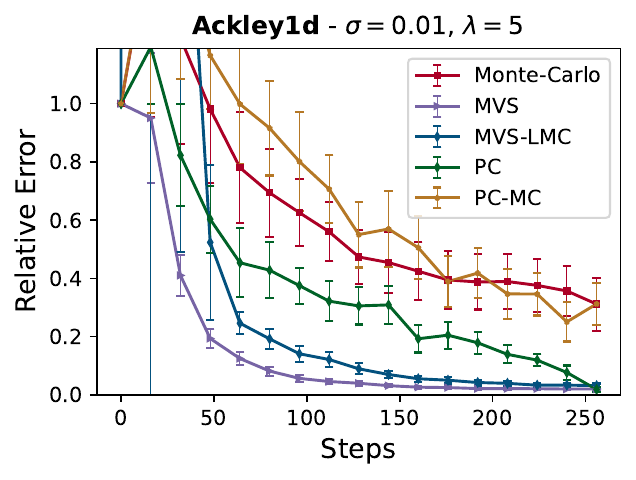}
        \includegraphics[width=\sizecustom]{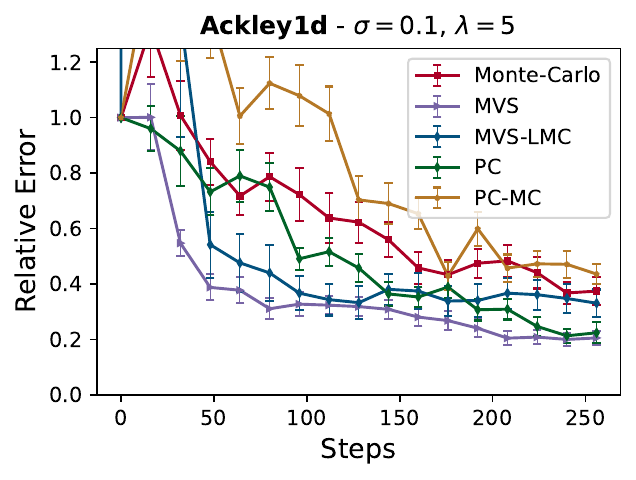}
    \end{subfigure}
    \vskip 0.2in
    \begin{subfigure}{\columnwidth}
        \centering
        \includegraphics[width=\sizecustom]{figs/analytic/zhou2-0.01-0.5.pdf}
        \includegraphics[width=\sizecustom]{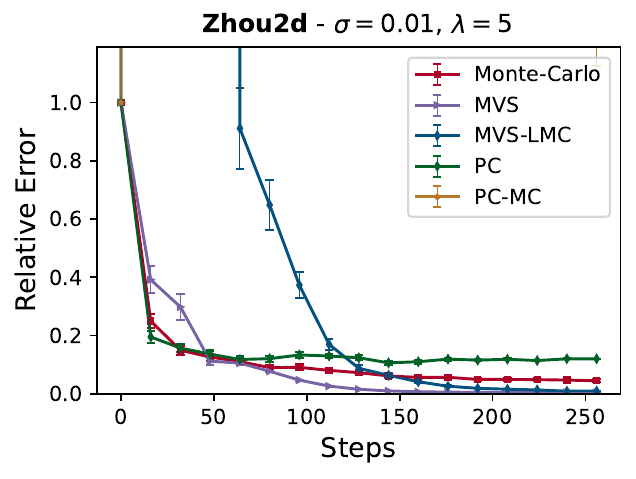}
        \includegraphics[width=\sizecustom]{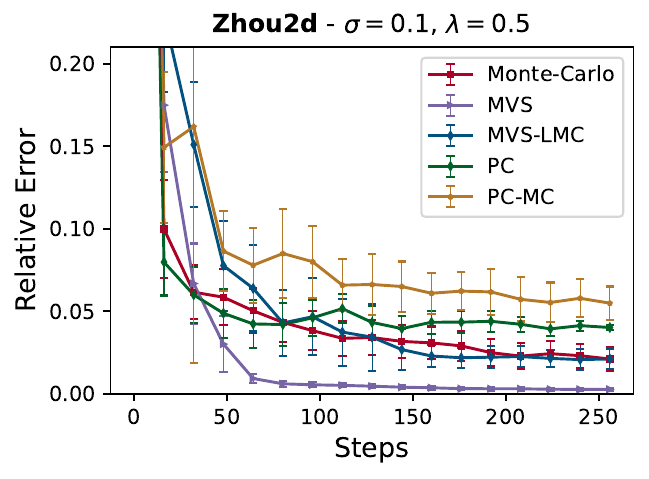}
    \end{subfigure}
    \vskip 0.2in
    \begin{subfigure}{\columnwidth}
        \centering
        \includegraphics[width=\sizecustom]{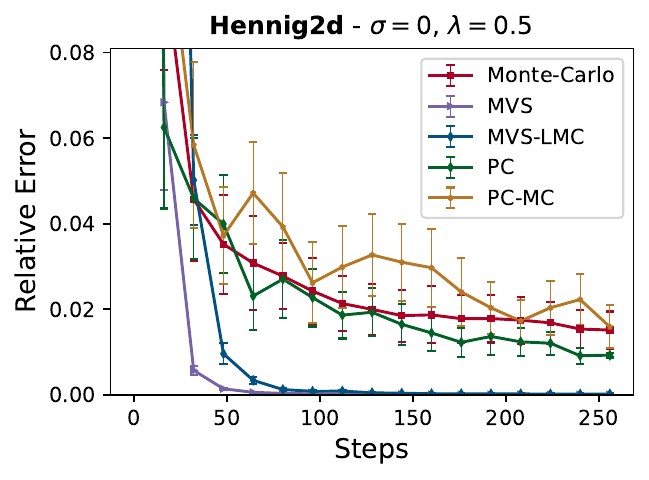}
        \includegraphics[width=\sizecustom]{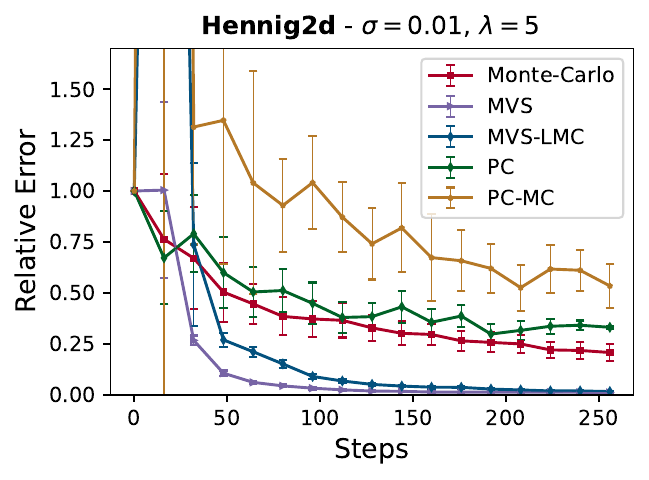}
        \includegraphics[width=\sizecustom]{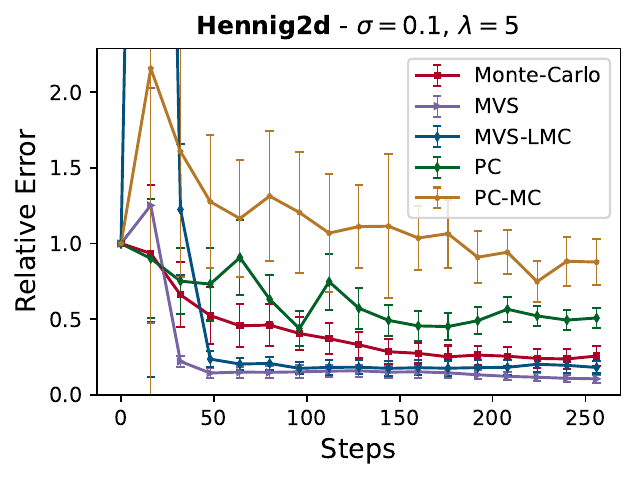}
    \end{subfigure}
    \vskip 0.2in
    \begin{subfigure}{\columnwidth}
        \centering
        \includegraphics[width=\sizecustom]{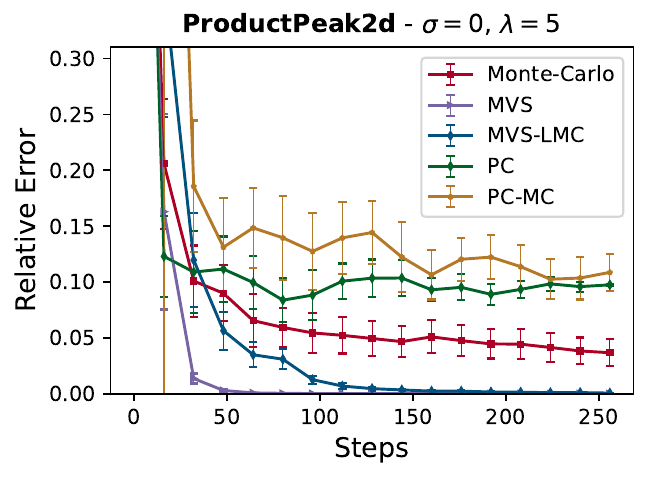}
        \includegraphics[width=\sizecustom]{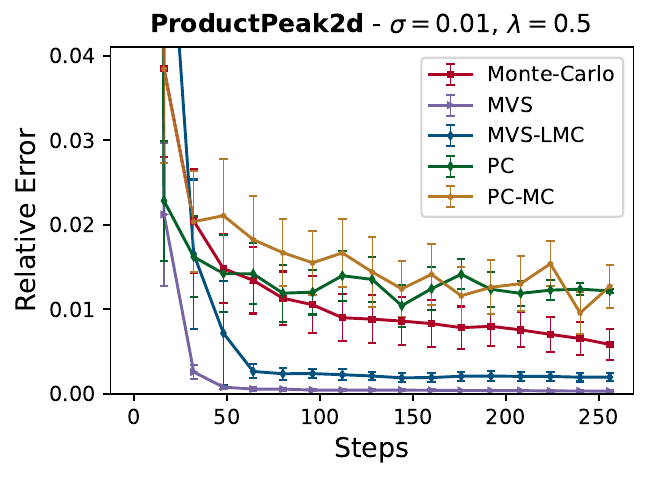}
        \includegraphics[width=\sizecustom]{figs/analytic/prod2-0.1-5.pdf}
    \end{subfigure}
    \caption{Additional results for analytic functions. \label{fig:add_analytic}}
\end{figure}

\end{document}